%% file: paper.tex
\documentclass[twoside]{article}

\usepackage[accepted]{aistats2022}

\usepackage{adjustbox}
\usepackage{algorithm}
\usepackage{amsfonts}
\usepackage{amsmath}
\usepackage{amssymb}
\usepackage{amsthm}
\usepackage{booktabs}
\usepackage{cancel}
\usepackage{cellspace}
\usepackage{csquotes}
\usepackage{comment}
\usepackage[inline]{enumitem}
\usepackage[mathscr]{eucal}
\usepackage{framed}
\usepackage{lipsum}
\usepackage{mathtools}
\usepackage{multirow}
\usepackage[round]{natbib}
\usepackage[subrefformat=parens]{subcaption}
\usepackage{thmtools}
\usepackage{tikz}

\usepackage[hidelinks]{hyperref}
\usepackage{doi}

\usepackage[noend]{algpseudocode}

\usepackage{cleveref}
\crefname{algorithm}{Alg.}{Algs.}
\crefname{appendix}{Appx.}{Appxs.}
\crefname{example}{Example}{Examples}
\crefname{equation}{}{}        
\Crefname{equation}{Eq.}{Eqs.} 
\crefname{figure}{Fig.}{Figs.}
\crefname{proposition}{Prop.}{Props.}
\crefname{corollary}{Corr.}{Corrs.}
\crefname{section}{Sec.}{Secs.}
\crefname{table}{Tbl.}{Tables}

\usetikzlibrary{calc}
\usetikzlibrary{positioning}
\usetikzlibrary{decorations.pathreplacing}

\declaretheorem[style=plain,numberwithin=section]{theorem}
\declaretheorem[style=plain,sibling=theorem]{corollary}
\declaretheorem[style=plain,sibling=theorem]{proposition}

\declaretheorem[style=definition,sibling=theorem,qed={\lower-0.3ex\hbox{$\triangleleft$}}]{example}

\newcommand{\defas}{\coloneqq}
\newcommand{\diff}{\mathop{}\!\mathrm{d}}
\newcommand{\set}[1]{\lbrace{#1}\rbrace}

\DeclareMathOperator{\SExp}{\mathbb{E}}
\newcommand{\E}[1]{\SExp\left[#1\right]}

\renewcommand{\H}[1]{H(#1)}

\DeclareMathOperator{\SVar}{\mathrm{Var}}
\newcommand{\V}[1]{\SVar\left[#1\right]}

\DeclareMathOperator{\SKL}{D_\mathrm{KL}}
\newcommand{\KL}[2]{\SKL\left[#1 {\mid\mid} #2  \right]}

\makeatletter
\newcounter{algorithmicH}
\let\oldalgorithmic\algorithmic
\renewcommand{\algorithmic}{%
  \stepcounter{algorithmicH}
  \oldalgorithmic}
\renewcommand{\theHALG@line}{ALG@line.\thealgorithmicH.\arabic{ALG@line}}
\makeatother

\begin{document}

%
\runningtitle{Estimators of Entropy and Information via Inference in Probabilistic Models}

%
\twocolumn[
  \aistatstitle{Estimators of Entropy and Information via Inference \\ in Probabilistic Models}
  \aistatsauthor{Feras A.~Saad \And Marco Cusumano-Towner \And Vikash K.~Mansinghka }
  \aistatsaddress{MIT \And MIT \And MIT }
]

\begin{abstract}
Estimating information-theoretic quantities such as entropy and mutual
information is central to many problems in statistics and machine
learning, but challenging in high dimensions.
This paper presents \textit{estimators of entropy via inference} (EEVI),
which deliver upper and lower bounds on many information quantities for arbitrary
variables in a probabilistic generative model.
These estimators use importance sampling
with proposal distribution families that include amortized variational
inference and sequential Monte Carlo, which can be tailored to the
target model and used to squeeze true information values with high accuracy.
We present several theoretical properties of EEVI
and demonstrate scalability and efficacy on two problems from the medical domain:
\begin{enumerate*}[label=(\roman*)]
\item in an expert system for diagnosing liver disorders, we rank
medical tests according to how informative they are about latent
diseases, given a pattern of observed symptoms and patient attributes;
and
\item in a differential equation model of carbohydrate metabolism, we find
optimal times to take blood glucose measurements that maximize
information about a diabetic patient's insulin sensitivity, given
their meal and medication schedule.
\end{enumerate*}
\end{abstract}

\section{INTRODUCTION}

This paper studies the fundamental problem of estimating the Shannon entropy
$H(Y) \defas -\mathbb{E}[\log p(Y)]$ of a random element $Y$,
in situations where its marginal distribution involves an intractable
multidimensional integral over a known joint probability distribution:
\begin{align}
p(y) = \int_{\mathcal{X}}p(x,y) \diff{x} && (y \in \mathcal{Y}).
\label{eq:marginal-py}
\end{align}
In~\cref{eq:marginal-py}, the term $p(x,y)$ refers to a probabilistic generative
model that can be sampled from and whose joint density can be computed pointwise,
as is common in a broad class of
probabilistic systems that includes Bayesian networks~\citep{pearl1988},
deep generative models~\citep{kingma2019}, and generative
probabilistic programs~\citep{wingate2011}.
In this setting, a key challenge is that the expression
\begin{align}
\H{Y} = -\int_{\mathcal{Y}} \log\left[\int_{\mathcal{X}} p(x,y)\diff{x}\right] p(y) \diff{y}
\label{eq:entropy-double-integral}
\end{align}
contains an intractable integral inside the logarithm, which rules out
the unbiased simple Monte Carlo estimator $-1/n\sum_{i=1}^n \log p(Y_i)$
(for $Y_i \sim p(y)$, $1\,{\le}\,i\,{\le}\,n$).

To address these challenges, we develop a new class of
\textit{estimators of entropy via inference} (EEVI) that return
interval estimates of doubly intractable entropies as
in~\cref{eq:entropy-double-integral}.
EEVI uses auxiliary-variable importance sampling
constructs similar to those from pseudo-marginal
methods~\citep{andrieu2009} to first compute unbiased estimates of the
intractable quantities $p(y)$ and $1/p(y)$ for the inner integral.
Under the log transform, these estimates become lower
and upper bounds of $\log p(y)$, which are then
embedded in a simple Monte Carlo estimator for the outer integral
to form an interval estimate of $\H{Y}$.
In the limit of computation, the interval width can be driven to zero,
squeezing the true entropy value at a rate that depends on the
quality of the importance sampling proposal.

Our contribution is a family of entropy estimators that
\begin{enumerate}[label=(C\arabic*),topsep=0pt, leftmargin=*, itemsep=0pt, partopsep=0pt]
\item Apply to arbitrary random variables
  in any probability distribution that can be sampled from and whose
  full joint density is tractable; no marginals or conditionals need
  to be tractable.

\item Guarantee upper and lower bounds in expectation, which can
  be composed (\cref{fig:derived}) to squeeze many
  information quantities, e.g., \crefrange{eq:measure-ce}{eq:measure-dcor}.

\item Leverage a broad family of proposal distributions
  that includes both variational and Monte Carlo inference
  for increasing accuracy as a function of computational effort.
\end{enumerate}

The rest of the paper is organized as follows:
\cref{sec:overview} gives an overview of EEVI and
  explains how interval estimates of entropy
  can be composed to form interval estimates of several other information
  quantities, such as conditional mutual information and
  interaction information.
\cref{sec:aux-var}
  presents theoretical properties of importance sampling-based
  estimators of log marginal probabilities of the form given in~\cref{eq:marginal-py},
  and gives examples of inference-based
  variational and Monte Carlo auxiliary-variable proposals
  to deliver accurate upper and lower bounds.
\cref{sec:applications} illustrates the scalability and efficacy of
  EEVI for two optimal design tasks in a
  probabilistic expert system for diagnosing liver disorders and a
  dynamic model of carbohydrate metabolism in diabetic patients.
\cref{sec:related} discusses related work.

\input{figures/derived}

\section{OVERVIEW OF EEVI}
\label{sec:overview}

Suppose that $p(z_1, \dots, z_d)$ is a $d$-dimensional probability
density (with respect to an appropriate $\sigma$-finite measure)
such that it is possible to sample $(Z_1,\dots,Z_d) \sim p(z_1,\dots,z_d)$
and evaluate density values pointwise.
Let $A \subset \set{1,\dots,d}$ be a subset of indexes and let $Y \defas \set{Z_i, i \in A}$
and $X \defas \set{Z_i, i \notin A}$ be the corresponding partition of
variables in $Z$.
We aim to estimate the marginal entropy $\H{Y}$ as defined
in~\cref{eq:entropy-double-integral}, where $\mathcal{Y}$ and
$\mathcal{X}$ are the sets in which $Y$ and $X$ take values, respectively.
As the partition $A$ is arbitrary, neither the marginal densities
$p(x)$ and $p(y)$ nor conditional densities $p(x\,{\mid}\,y)$ and
$p(y\,{\mid}\,x)$ are necessarily tractable, posing a key
challenge for estimating $\H{Y}$.

Suppose momentarily that we can compute two
measurable real functions $w, w': \mathcal{U} \times \mathcal{Y} \to \mathbb{R}$
such that for some random variables $U, U'$ taking values in a set $\mathcal{U}$
and all $y \in \mathcal{Y}$ except for a $p$-measure zero set,
we have
\begin{align}
\E{w(U, y)} = p(y) && \E{w'(U', y)} = 1/p(y).
\label{eq:expect-lu}
\end{align}
If $w$ and $w'$ are nonnegative
almost everywhere, then concavity of $\log$ and
Jensen's inequality gives bounds
\begin{align}
\E{\log w(U, y)} &\le \log \E{w(U, y)} = \log p(y), \label{eq:ebound-lower}\\
\E{\log w'(U', y)}   &\le \log \E{w'(U',y)} = -\log p(y), \label{eq:ebound-upper}
\end{align}
which together imply that
\begin{align}
\E{\log w(U,y)} \le \log p(y) \le \E{-\log w'(U',y)}.
\label{eq:logpy-sandwich}
\end{align}
If the real functions $y \mapsto \E{\log w(U, y)}$ and
$y \mapsto -\E{\log w'(U', y)}$ defined on $\mathcal{Y}$ are themselves both
measurable
then monotonicity of expectation gives
\begin{flalign}
\E{\log w(U,Y)} \le \E{\log p(Y)} \le \E{-\log w'(U', Y)} \hspace{-1.5cm} &&
\label{eq:logpY-sandwich}
\end{flalign}
The two expectations in \eqref{eq:logpY-sandwich} that squeeze $\E{\log p(Y)}$
can now be estimated via unbiased Monte Carlo
\begin{align}
\mathcal{L}_{n,m} &\defas
  \frac{1}{n} \sum_{i=1}^n \left[ \frac{1}{m} \sum_{j=1}^m \log w(U_{ij}, Y_i) \right]
  \label{eq:Hl-generic}, \\
\mathcal{T}_{n,m} &\defas
  \frac{1}{n} \sum_{i=1}^n \left[ \frac{1}{m} \sum_{j=1}^m -\log w'(U'_{ij}, Y'_i) \right]
  \label{eq:Hu-generic},
\end{align}
where $Y_i, Y'_i$ are identically distributed to $Y$; $U_{ij}$ identically
to $U$; and $U'_{ij}$ identically to $U'$ $(i=1,\dots,n; j=1,\dots,m)$.
Letting $\check{H}_Y \defas -\mathcal{T}_{n,m}$ and $\hat{H}_Y \defas -\mathcal{L}_{n,m}$,
\cref{eq:logpY-sandwich} implies that
the means of $\check{H}_Y$ and $\hat{H}_Y$ satisfy
\begin{align}
\mathbb{E}[\check{H}_Y] \le \H{Y} \le \mathbb{E}[\hat{H}_Y].
\label{eq:Hy-sandwich}
\end{align}
If using i.i.d.\ samples in~\cref{eq:Hl-generic,eq:Hu-generic},
under mild conditions the central limit theorem and \cref{eq:Hy-sandwich}
imply the interval estimator $[\check{H}_Y, \hat{H}_Y]$ has coverage probability
\begin{align*}
\Pr[ \check{H}_Y \le H(Y) \le \hat{H}_Y] \approx
  \Phi\left(\sqrt{t}\check{B}/{\check{\sigma}}\right)
  \Phi\left(\sqrt{t}\hat{B}/\hat{\sigma}\right),
\end{align*}
where $\Phi$ is the standard normal CDF; $t = nm$;
$\check{B} \defas \E{-\log w'(U',Y)} - \E{\log p(Y)}$
and
$\hat{B} \defas \E{\log p(Y)} - \E{\log w(U,Y)}$
are the biases in~\cref{eq:Hy-sandwich};
and $\check{\sigma}$ and $\hat{\sigma}$ are the standard deviations
of $\log w'(U',Y)$ and $\log w(U, Y)$.
So far we have assumed access to functions $w$ and $w'$ that
satisfy~\cref{eq:expect-lu}: \cref{sec:aux-var} shows how they can be
constructed via importance sampling.

\subsection{Extending entropy bounds to additional information-theoretic quantities}
\label{sec:overview-derived}

The lower and upper bounds on entropy in~\cref{eq:Hy-sandwich} can be composed to bound
several derived information-theoretic
quantities that measure the degree of relationship between arbitrary
subcollections of variables in a model, possibly conditioned on others
(\cref{fig:derived}).
Letting $\H{A} \defas \H{\set{Z_i, i \in A}}$ for $A \subset [d]$,
by adding and subtracting upper and lower bounds on $\H{A}$ we
can also build interval estimators of:
\begin{flalign}
&\,\scalebox{.8}{$\bullet$}\, \textit{conditional entropy}~\textrm{\citep{shannon1948}} && \notag \\
&\, \H{{A_1} \mid {A_2}} \defas \H{A_1 \cup A_2} - \H{A_2} \label{eq:measure-ce} \\
&\,\scalebox{.8}{$\bullet$}\, \textit{conditional mutual information}~\textrm{\citep{shannon1948}}  && \notag \\
&\, I({A_1} : {A_2} \mid {A_0}) \defas \H{A_1 \mid A_0}  - \H{A_1 \mid A_2, A_0}  \label{eq:measure-cmi} \hspace{-1cm} \\
&\,\scalebox{.8}{$\bullet$}\, \textit{conditional total correlation}~\textrm{\citep{watanbe1960}}  && \notag \\
&\, \begin{aligned}[t]
  C(\set{A_i}_{i=1}^n {\mid} A_0) &\defas \textstyle\sum\limits_{i=1}^{n}\H{A_i {\mid} A_0}
  - \H{\mathop{\cup}\limits_{i=1}^{n}A_i \mid A_0} \label{eq:measure-tcor}
\end{aligned} \hspace{-1cm}\\
&\,\scalebox{.8}{$\bullet$}\, \textit{conditional interaction information}~\textrm{\citep{ting1962}} && \notag \\
&\, T(\set{A_i}_{i=1}^n \mid A_0)
  \defas \textstyle\sum_{S{\subset}[n]} -1^{|S|} \H{\cup_{i \in S}A_i\,{\mid}\,A_0} \hspace{-1cm}
  \label{eq:measure-intinf} \\
&\,\scalebox{.8}{$\bullet$}\, \textit{conditional dual correlation}~\textrm{\citep{han1978}}  && \notag \\
&\, \begin{aligned}[t]
  D(\set{A_i}_{i=1}^n \mid A_0) &\defas H(\cup_{i=1}^nA_i \mid A_0) \\
  &- \textstyle\sum_{i=1}^n H(A_i \mid \cup_{\substack{j=0, j\ne i}}^nA_i)
  \label{eq:measure-dcor}
\end{aligned}
\end{flalign}
\section{SAMPLING BOUNDS ON LOG MARGINAL PROBABILITIES}
\label{sec:aux-var}

\textbf{Importance sampling in log space.}
Implementing the estimators $\hat{H}_Y, \check{H}_Y$ in~\cref{eq:Hy-sandwich} requires
functions $w, w'$ that satisfy~\cref{eq:expect-lu}.
Our starting point is an identity from importance sampling.
Let $h$ and $g$ be two probability densities on a common set $\mathcal{X}$
such that $h$ is absolutely continuous with respect to $g$ (written $h \ll g$);
i.e., $\int_A g(x) \diff{x} = 0 \implies \int_A h(x) \diff{x} = 0$ for all
measurable $A$.
Suppose
$h(x) = \tilde{h}(x)/Z_h$, $g(x) = \tilde{g}(x)/Z_g$
are only known up to normalizing constants.
Then, for $X\sim g$,
\begin{align}
\E{\tilde{h}(X)/\tilde{g}(X)} = Z_h/Z_g.
\label{eq:is-ratio-unbisaed}
\end{align}
(All proofs in~\cref{appx:proofs}.)
Under log transform, the ratio in~\cref{eq:is-ratio-unbisaed}
is a lower bound on $\log(Z_h/Z_g)$ in expectation with a gap equal to
the KL divergence from $h$ to $g$:
\begin{flalign}
\E{ \log\left({\tilde{h}(X)}/{\tilde{g}(X)} \right) }
  = \log \left({Z_h}/{Z_g}\right) - \KL{g}{h}\hspace{-.5ex}.
  \hspace{-1cm} &&
\label{eq:is-log-ratio-biased}
\end{flalign}
\Cref{eq:is-log-ratio-biased} does not require $h \ll g$.
However, the expectation is well-defined only if $g \ll h$
and is finite only if $\KL{g}{h} < \infty$.
Moreover, the variance
\begin{align}
\begin{aligned}[m]
&\mathrm{Var}\left[\log(\tilde{h}(X)/\tilde{g}(X)) \right] \\
&\qquad= \E{ \log^2(h(X)/g(X)) } - (\KL{g}{h})^2
\end{aligned}
\label{eq:is-log-ratio-var}
\end{align}
is finite only if $\KL{g}{h} < \infty$ and $\log^2(h(X)/g(X))$ has finite expectation.
Applying Markov's inequality to~\cref{eq:is-ratio-unbisaed} gives a right tail bound
for $\log \tilde{h}(X)/\tilde{g}(X)$:
\begin{align}
&\Pr\left[\tilde{h}(X)/\tilde{g}(X) \ge e^t(Z_h/Z_g)\right] \le e^{-t} \label{eq:is-log-ratio-tail} \\
&\implies \Pr\left[\log(\tilde{h}(X)/\tilde{g}(X)) \ge t + \log(Z_h/Z_g)\right] \le e^{-t}, \notag
\end{align}
for any $t > 0$.
The mean absolute deviation satisfies
\begin{align}
\E{\left\lvert \log(\tilde{h}(X)/\tilde{g}(X)) - \mu \right\rvert}
  \le 2 + 2\KL{g}{h},
\label{eq:is-log-ratio-ead}
\end{align}
where $\mu \defas \mathbb{E}[\log(\tilde{h}(X)/\tilde{g}(X))]$;
i.e., it is upper bounded by two plus twice the
bias in~\cref{eq:is-log-ratio-biased}, which decreases as $g$
more closely matches $h$.

\input{figures/alg-eevi}

\textbf{Interval estimators of entropy.}
Recalling the distribution $p(x,y)$ from~\cref{sec:overview},
suppose that $q(x; y)$ and $q'(x;y)$ are normalized proposal densities
over $\mathcal{X}$ parameterized by $\mathcal{Y}$.
From~\cref{eq:is-ratio-unbisaed}, for fixed $y \in \mathcal{Y}$,
setting $\tilde{h}(x) = p(x,y)$, $g(x) = q(x;y)$ gives an unbiased
estimate of $Z_h \equiv p(y)$;
and setting $h(x) = q'(x;y)$, $\tilde{g}(x) = p(x,y)$
gives an unbiased estimate of $1/Z_g \equiv 1/p(y)$:
\begin{align}
\E{\frac{p(X,y)}{q(X; y)}} = p(y), && \E{\frac{q'(X'; y)}{p(X',y)}} = 1/p(y),
\label{eq:laystall}
\end{align}
as needed for~\cref{eq:expect-lu}, by defining $w(x,y) \defas
p(x,y)/q(x;y)$ and $w'(x,y) \defas q'(x;y)/p(x,y)$ and letting $X \sim q(x;y)$
be $U$ and $X' \sim p(x \mid y)$ be $U'$.
Then~\cref{eq:Hl-generic} yields the Monte Carlo \emph{upper}
bound $\hat{H}_Y$ in~\cref{eq:Hy-sandwich};
and \cref{eq:Hu-generic} yields the Monte Carlo \emph{lower} bound
$\check{H}_Y$ in~\cref{eq:Hy-sandwich}.
While sampling $X' \sim p(x \mid y)$ given a fixed value $y$ is typically intractable
under our assumptions on $p$, since $\H{Y}$ is the expectation of random values
$-\log p(Y)$ for $Y \sim p$, it suffices to use joint samples $(X',Y') \sim p(x,y)$ to obtain the lower bound, since
\begin{align*}
&\int_{\mathcal{X}\times\mathcal{Y}}\log\left[q'(x;y)/p(x,y)\right]p(x,y) \diff{x}\diff{y} \\
&= \int_{\mathcal{Y}}\left[\int_{\mathcal{X}} \log\left[q'(x;y)/p(x,y)\right]p(x\mid y)\diff{x}\right]p(y) \diff{y} \\
&\le \int_{\mathcal{Y}}\left[-\log p(y)\right]p(y) \diff{y} = \H{Y},
\end{align*}
where the second line follows from~\citet[Thms.~B.46, B.52]{schervish1995}
and third line from~\cref{eq:is-log-ratio-biased}.
Given an initial sample $(X'_1, Y) \sim p(x,y)$,
an additional $m-1$ samples $\set{X'_2, \dots, X'_m}$ from $p(x \mid Y')$
to use for $\mathcal{T}_{n,m}$ in~\cref{eq:Hu-generic}
can be obtained by simulating a Markov
chain initialized at $X'_1$ that leaves $p(x \mid Y')$ invariant,
which will reduce $\V{\mathcal{T}_{n,m}}$ iff $\Pr[X'_i \ne X'_1] > 0$ for some $i$.
\subsection{Constructing accurate proposals}
\label{sec:aux-var-smc}

Estimators of normalizing constants and their inverses in direct
space as in~\cref{eq:laystall} can suffer from notoriously high variance,
especially when using proposals that do not closely match the
target~\citep{neal2008}.
However, \cref{eq:is-log-ratio-tail} suggests that log space estimators
can be more stable and, by~\cref{eq:is-log-ratio-biased,eq:is-log-ratio-ead},
the quality of the entropy bounds obtained via importance sampling
depends on constructing proposal distributions $q(x;y)$ and
$q'(x;y)$ that have small biases in expectation (over $Y \sim p(y)$):
\begin{align}
\mathbb{E}[\hat{H}_Y] - H(Y) &= \E{\KL{q(x; Y)}{p(x \mid Y)}} \label{eq:gap-upper}, \\
H(Y) - \mathbb{E}[\check{H}_Y] &= \E{\KL{p(x \mid Y)}{q'(x; Y)}} \label{eq:gap-lower}.
\end{align}
We next discuss two approaches to constructing accurate proposals
using probabilistic inference algorithms.

\textbf{Amortized variational inference.}
One approach to constructing accurate proposals $q(x; y)$ and $q'(x; y)$ in \cref{eq:laystall} is
to use a dataset $\set{(x_i,y_i)}_{i=1}^n$ simulated from $p$
to variationally train recognition networks $q_\phi(x; y)$ and $q'_\varphi(x; y)$
that each specify a family of distributions over $\mathcal{X}$, parametrized by $y \in \mathcal{Y}$ and $\phi$, $\varphi$, respectively.
Training $q$ via ``exclusive'' amortized variational inference, as in variational autoencoders~\citep{kingma2013}, minimizes both the bias
and an upper bound on the mean absolute deviation (MAD) of
the $\check{H}_Y$ in \cref{alg:entropy-upper-general}.
Similarly, training $q'$ via ``inclusive'' amortized variational inference, as in the ``sleep'' phase of the wake-sleep algorithm~\citep{hinton1995}, minimizes both the bias
and an upper bound on the MAD of $\hat{H}_Y$ in \cref{alg:entropy-lower-general}.
Thus, EEVI can leverage advances in training neural networks via stochastic gradient descent to improve estimation accuracy (refer to \cref{fig:mvn}).

\textbf{Auxiliary-variable Monte Carlo.}
Another approach to constructing accurate proposals,
which can be composed with variational learning~\citep{salimans2015},
is Monte Carlo methods such as annealed
importance sampling~\citep[AIS;][]{neal2001} and sequential Monte
Carlo~\citep[SMC;][]{del2006} that define proposal distributions on extended
state-spaces and yield state-of-the-art estimates of normalizing constants.
More specifically, these proposals $q(v,x; y)$ are defined
over an extended space $\mathcal{V} \times \mathcal{X}$, such that the
marginal $q(x; y) = \int_{\mathcal{V}}q(v,x;y) \diff{v}$ is an
integral over all auxiliary random variables sampled by $q$.
As the ratios in~\cref{eq:laystall} can no longer be evaluated,
we define a tractable ``auxiliary proposal distribution''
$r(v; x, y)$ over $\mathcal{V}$ parameterized by
$\mathcal{X} \times \mathcal{Y}$ (\cref{fig:distributions})
such that
\begin{align}
w(v,x,y)  &\defas \left[p(x,y)r(v; x,y)\right]/{q(v, x; y)}, \label{eq:w-extended-lo} \\
w'(v,x,y) &\defas q'(v, x; y)/\left[p(x,y)r'(v; x, y)\right]. \label{eq:w-extended-hi}
\end{align}
By \cref{eq:is-ratio-unbisaed}, these weight functions
satisfy~\cref{eq:expect-lu} by letting $(V,X) \sim q(v,x;y)$ serve as
$U$ and $(X',V') \sim p(x\mid y)r(v;x,y)$ as $U'$.
From~\cref{eq:is-log-ratio-biased}, the gap when lower bounding $\log p(y)$ using
these extended-space weights $w$ and $w'$
now accounts for the accuracy of the auxiliary proposals $r$ and $r'$
(illustrated in \cref{appx:eevi-variants}, \cref{fig:gaps}):
\begin{align}
\mathbb{E}[\hat{H}_Y] - H(Y) &\begin{aligned}[t]
    &= \E{\KL{q(x; Y)}{p(x | Y)}} \\
    &\hspace{-.5cm}+ \E{\KL{q(v \mid X; Y)}{r(v; X, Y)}}, 
\end{aligned} \label{eq:gap-aux-upper}\\
H(Y) - \mathbb{E}[\check{H}_Y]&\begin{aligned}[t]
  &= \E{\KL{p(x \mid Y)}{q'(x; Y)}} \\
  &\hspace{-.5cm}+ \E{\KL{r(v'; X', Y)}{q'(v \mid X'; Y)}}. \hspace{-1cm}
\end{aligned} \label{eq:gap-aux-lower}
\end{align}
\cref{alg:entropy-upper-general,alg:entropy-lower-general}
show interval estimators $[\check{H}_Y, \hat{H}_Y]$ that
implement~\cref{eq:Hu-generic,eq:Hl-generic} using the
extended weights $w$ and $w'$
in~\cref{eq:w-extended-lo,eq:w-extended-hi}.
Proposals without auxiliary variables are a special case,
where $\mathcal{V} = \set{\omega}$ is a singleton and
$r(v;x,y) = \delta(v; \omega)$
(\cref{appx:eevi-variants}, \cref{alg:entropy-lower-bound-basic,alg:entropy-upper-bound-basic}).

\input{figures/mvn}

\begin{example}[Sampling-Importance Resampling]
\label{example:sir}
To fix ideas, consider a base proposal $q_0(x;y)$ that has no
auxiliary variables, which may have been hand
constructed or trained variationally.
The proposal $q_0$ can be embedded in a sampling-importance-resampling (SIR)
scheme that generates $P$ variables $x_{1:P}$ i.i.d.\ from $q_0$
and a selection index $k$ taking value $i$ with relative probability
$p(x_i,y)/q_0(x_i; y)$ (i.e., the auxiliary variables $v \defas (x_{1:P}, k)$),
then sets $x \gets x_k$:
\begin{align*}
q((x_{1:P}, k), x; y) &= \prod_{j=1}^P q_0(x_j) \left[\frac{\frac{p(x_k,y)}{q_0(x_k;y)}}{\sum_{i=1}^P \frac{p(x_i,y)}{q_0(x_i;y)}}\right] \delta(x; x_k). \\
\shortintertext{The task of the auxiliary proposal $r((x_{1:P}, k); x, y)$ is to infer $v$ for an $(x,y)$ pair as follows:}
r((x_{1:P}, k); x, y) &= \prod_{\substack{j=1 \\ j \ne k}}^P q_0(x_j; y) \left[\frac{1}{P}\right] \delta(x_k; x).
\end{align*}
The weight~\cref{eq:w-extended-lo} is then precisely the usual SIR estimate
of the marginal density of $y$,
\begin{align}
\frac{p(x,y)r(v;x,y)}{q(v,x;y)} = \frac{1}{P}\sum_{j=1}^P \frac{p(x_j,y)}{q_0(x_j;y)}.
\label{eq:sir-p}
\end{align}
By \citet[Thm.~1]{burda2016},
if $p(x,y)/q_0(x;y)$ is bounded then
$\KL{q(v,x;y)}{p(x\mid y)r(x)} \to 0$ (bias in~\cref{eq:is-log-ratio-biased})
as $P\to \infty$.
Refer to~\cref{fig:mvn} for an illustration
and~\cref{appx:eevi-variants}, \cref{alg:entropy-lower-bound-sir,alg:entropy-upper-bound-sir}
for EEVI with SIR.
\end{example}

\input{figures/alg-csmc}
\input{figures/hepar2}

\begin{example}[Generalized Sequential Monte Carlo]
\label{example:smc}
The SIR proposal from~\cref{example:sir} can be generalized to the
setting of a sequence
$\set{\tilde{p}_t(x;y)}_{t=0}^T$
with $T$ intermediate (unnormalized) target densities
such that $\tilde{p}_T(x,y) = p(x,y)$,
which may represent partial posteriors in particle
filtering for temporal models~\citep{doucet2011} or tempered distributions as in
AIS~\citep{neal2001} and sequential Bayesian
updating~\citep{del2006} for static models.
\cref{alg:smc-forward} shows the proposal $q(v,x;y)$ from
a run of SMC
with initial kernel $q_0(x_0; y)$;
forward kernels $q_t(x_t; x_{t-1}, y)$ ($t=1 \dots T$);
backward kernels $l_t(x_t; x_{t+1}, y)$ ($t=0 \dots T-1$);
and $P$ particles.
Here, $v \defas (I_T, a^{1:P}_{1:T}, x^{1:P}_{0:T})$ contains all
auxiliary variables and $x \sim \delta({x^{I_T}_T})$ is the selected
final particle.
\cref{alg:smc-reverse} shows the auxiliary proposal $r(v;x,y)$,
which infers $v$ given $(x,y)$ using generalized ``conditional SMC''~\citep{andrieu2010,towner2017}.
Simplifying the weights~\cref{eq:w-extended-lo,eq:w-extended-hi} gives
\begin{align*}
w(v,x,y)= \prod_{t=0}^T\left[ \frac{1}{P}\sum_{j=1}^P w^j_t \right],
\; w'(v,x,y) = \frac{1}{w(v,x,y)},
\end{align*}
where $w^j_t$ terms are defined
in~\cref{alg:smc-forward} \cref{algline:smc-forward-weight} for $w(v,x,y)$
and~\cref{alg:smc-reverse} \cref{algline:smc-reverse-weight} for $w'(v,x,y)$.
It is also possible to compose SMC with SIR (\cref{appx:eevi-variants}, \cref{alg:entropy-lower-bound-sir-nested,alg:entropy-upper-bound-sir-nested})
and learn $q_t$ variationally~\citep{maddison2017}.
\end{example}

\section{APPLICATIONS}
\label{sec:applications}

We applied EEVI to two data acquisition problems:
In \cref{sec:applications-hepar}, we rank medical tests in an expert
system for diagnosing liver disorders according to their conditional
mutual information with diseases of interest, given a pattern of
symptoms and patient attributes.
In \cref{sec:applications-diabetes}, we analyze a dynamic insulin
model to compute optimal times to take blood glucose measurements that
maximize information about a patient's insulin sensitivity, given
their insulin and meal schedule.
Experiments were conducted with the Gen~\citep{towner2019gen}
probabilistic programming system (see code supplement).

\subsection{HEPAR Liver Disease Network}
\label{sec:applications-hepar}

HEPAR~\citep{lucas1989} is a medical expert system that helps
physicians diagnose complex disorders in the liver and biliary tract.
We analyze the Bayesian network variant of HEPAR from~\citet{hepar2003}
shown in~\cref{fig:hepar-network},
which contains nodes representing a patient's
\begin{enumerate*}[label=(\roman*)]
\item attributes, such as age and obesity;
\item latent liver diseases, such as PBC and cirrhosis; and
\item symptoms, such as nausea and blood pressure.
\end{enumerate*}
We consider the following inference problem: {\itshape Given a patient
with a set $\set{o_i}$ of observed attributes and symptoms, which
medical tests $\set{t_j}$ for symptoms should the physician conduct
to maximize information about the absence or presence of a disease
$d$?}
We formalize the problem as ranking the tests by decreasing
conditional mutual information (CMI) with $d$:
\begin{align}
I(d : t \mid \set{o_i}) = H(d \mid \set{o_i}) - H(d \mid t, \set{o_i}).
\label{eq:hepar-cmi}
\end{align}
Since $\H{d\,{\mid}\,\set{o_i}}$ is constant,
it is more computationally efficient (and
equivalent) to rank tests $\set{t_j}$ by increasing conditional entropy
$\H{d\,{\mid}\,t_j, \set{o_i}}$ as in \cref{eq:measure-ce}.

\input{figures/hepar2-runtime}

\textbf{Results.} \cref{fig:hepar-network} shows a setting of 20
observed nodes (red), 31 medical tests for symptoms (yellow), and two liver diseases (blue).
\cref{tab:hepar-cirrhosis-ranking} shows
the top 10 tests ranked by $\H{d \mid t_j, \set{o_i}}$
In~\cref{tab:hepar-pbc-ranking,tab:hepar-cirrhosis-ranking} the first
and second columns show the top 10 most informative tests and
conditional entropy values for PBC and cirrhosis diseases,
respectively.
The conditional entropies are computed as the midpoint of interval estimates from EEVI
(\cref{alg:entropy-upper-general,alg:entropy-lower-general}),
using SIR auxiliary variable proposal (\cref{example:sir}) with
ancestral sampling base proposal and number of particles $P$ that
drives the interval width to below $10^{-3}$ nats.
To assess how useful these rankings might be to a physician, the
third columns in~\cref{tab:hepar-pbc-ranking,tab:hepar-cirrhosis-ranking}
shows median prediction errors for each disease in 5000 random patients
obtained by conditioning on a given test.
For each test $t_k$, prediction error is defined as the log loss
between the posterior $p(d \mid t_k, \set{o_i})$ and
the ground-truth label; for reference, prediction errors using $p(d \mid
\set{o_i})$ (i.e., no test) and $p(d)$ (i.e., no test or obs) are also
shown.
The results confirm that tests with higher information values
correlate with lower errors and that conditional entropy estimates
from EEVI can serve as a useful decision-making tool in this expert
system.

\textbf{Runtime.} \cref{fig:hepar-kraskov} compares runtime vs.\
accuracy profiles of EEVI to the nonparametric estimator
of~\citet{kraskov2004} for estimating the joint entropy
$\H{\set{o_i}_{i=1}^k}$ of $k$-dimensional random variables in the HEPAR
network ($k=10$, $20$, $40$).
For these queries, upper and lower bounds from EEVI converge between
1--5 seconds for each $k$.
In contrast, the nonparametric estimator converges slower, as it is
``model-free'' and estimates log probabilities from simulated data
without leveraging model structure.
The plots also highlight a key feature of EEVI: the width of the interval
quantifies the accuracy of the estimate at any given
level of computation and can squeeze the true value,
whereas the nonparametric estimator provides point
estimates (typically lower bounds) whose accuracy at a given level
of computation is unknown.

\textbf{Variance Control.} Following~\cref{eq:measure-ce}, bounds on the
conditional entropy $\H{d \mid t_j, \set{o_i}}$ in~\cref{fig:hepar}
are a difference of two marginal entropy bounds (see also~\cref{fig:derived}):
\begin{flalign}
H^{\rm lb}({d \,{\mid}\, t_j, \set{o_i}}) &= H^{\rm lb}({d, t_j, \set{o_i}})\,{-}\,H^{\rm ub}({t_j, \set{o_i}}), \hspace{-5cm} && \label{eq:ce-diff-lb} \\
H^{\rm ub}({d \,{\mid}\, t_j, \set{o_i}}) &= H^{\rm ub}({d, t_j, \set{o_i}})\,{-}\,H^{\rm lb}({t_j, \set{o_i}}). \hspace{-5cm} && \label{eq:ce-diff-ub}
\end{flalign}
\cref{fig:hepar-variance-boxplot} shows 18 realizations of interval
estimates of conditional entropy using shared samples (non-i.i.d.\
Estimator) and independent samples (i.i.d.\ Estimator) to bound the
marginal entropies.
The non-i.i.d.\ Estimator has lower variance as compared to the i.i.d.\ Estimator
and ensures that the realized lower bound is smaller than the realized upper bound.
\cref{fig:hepar-variance-scatter} explains this behavior in terms of the
correlation of the random weights~\cref{eq:w-extended-lo,eq:w-extended-hi}
used to estimate the four marginal entropies in~\cref{eq:ce-diff-lb,eq:ce-diff-ub}
(recall that $\mathrm{Var}[A - B] = \mathrm{Var}[A] + \mathrm{Var}[B] - 2\mathrm{Cov}[A, B]$
for any pair of real random variables $A$ and $B$).
We recommend that practitioners empirically assess the variance
and width of the interval estimators as in~\cref{fig:hepar-variance-scatter},
as well as inspect scatter plots of the weights when adding/subtracting bounds
from EEVI to bounds on derived information measures.

\input{figures/hepar2-variance}
\input{figures/diabetes2}

\subsection{Dynamic Insulin Model for Diabetes}
\label{sec:applications-diabetes}

We applied EEVI to solve a data acquisition task in a
differential equation model of carbohydrate
metabolism~\citep{andreassen1991} used by physicians for insulin
adjustment in diabetic patients.
While most medications have two or three standard dosing options,
insulin dose is highly individualized to each patient.
Finding the correct dose often requires an iterative adjustment
process that accounts for the patient's insulin response as well as
their insulin sensitivities and changing clinical conditions.
A variety of commercial insulin management software such as Glytec and EndoTool
have been developed to aid clinicians with this nuanced and costly process.

In~\cref{fig:diabetes-network} the ``insulin sensitivity''
node (blue) is a global latent parameter that dictates how effectively the
patient converts released insulin (from e.g., oral medications or injections)
into biologically usable insulin.
The ``meal'' (green) and ``insulin release'' (yellow) nodes are
intervention variables based on the patient's victual and
medication intake over a 25-hour period.
At time $t$, the blood glucose level is a noisy function of
insulin sensitivity and the following variables at $t-1$: blood
glucose level, meal, insulin release, and 14 intermediate biological
latent variables (white nodes); see~\citet{andreassen1991}
for full details.

Suppose a diabetic patient is undergoing the insulin adjustment process
and the physician is interested in the following problem:
{\itshape Given the patient's meal and insulin release schedule, at
which pairs of times should blood glucose level be measured to
maximize information about insulin sensitivity?}
We formalize the problem as ranking pairs of times $(t_1, t_2)$ by
decreasing CMI values with insulin sensitivity
(for $0 \le t_1 < t_2 \le 24$), given the meal and insulin release schedule:
\begin{align*}
I(\textrm{ins sens}, (\mathrm{BG}_{t_1}, \mathrm{BG}_{t_2}) \mid \set{(\textrm{meal}_t, \textrm{ins rel}_t) }_{t=0}^{24}),
\end{align*}
where $\mathrm{BG}_{t}$ indicates blood glucose measured at time $t$
and ``ins sense'' the latent insulin sensitivity.
To handle the temporal model, we compute entropy bounds
using the SMC proposals
in~\cref{example:smc,alg:smc-forward,alg:smc-reverse}, with a number
of particles $P$ that squeezes the interval width to $10^{-2}$
nats.
\cref{fig:diabetes-cmi-1,fig:diabetes-cmi-2} show estimates of CMI for
all pairs of time points under two different scenarios of meal and
insulin release.
Assuming that the model of~\citet{andreassen1991} is accurate, the
optimal times from these heatmaps may represent valuable insights for
the physician, as insulin sensitivity relates to blood glucose level
through complex dynamics of carbohydrate metabolism that are
challenging and costly to assess heuristically.
Interval estimators of entropy enable quantitative analysis of
information values of time points by probabilistic inference in the
model for any meal and insulin schedule.

\section{RELATED WORK}
\label{sec:related}

Several works have studied variational bounds of information
measures in a known distribution~\citep{barber2004,alemi2018,foster2019,poole2019}.
Unlike EEVI, these previous estimators do not not apply to arbitrary
subsets of random variables in a generative model and do not
return interval estimates.
Thus, the estimators in this paper can compute two-sided bounds in
settings not handled previously.
For example, upper bounding $I(X\,{:}\,Y)$ in \citet[Fig.~1]{poole2019} and \citet[Eq.~(9)]{foster2019}
assumes $p(y\,{\mid}\,x)$ is tractable.
In contrast, \cref{alg:entropy-upper-general,alg:entropy-lower-general}
in this paper can be used to compute interval estimates of $I(X\,{:}\,Y)$ even when
$p(y\,{\mid}\,x)$ and $p(x\,{\mid}\,y)$ are
intractable; see applications in~\cref{sec:applications}.

Several works have developed nonparametric entropy estimators given
i.i.d.\ data from an unknown distribution~\citep{kozchenkko1987,paninski2003,kraskov2004,cruz2009,belghazi2018,goldfeld2020}.
This work assumes that the distribution is known.
With our assumptions on the target model $p$, nonparametric
estimators can often be used in model-based settings by applying them
to i.i.d.\ data simulated from the model.
A drawback of using nonparametric estimators in this way, however, is
that they ignore the known model structure: \cref{fig:hepar-kraskov}
suggests that EEVI scales better, because the model structure is used
to build a suitable proposal.
A second advantage of EEVI is that the interval width indicates the
quality of the estimate at a given level of computational effort,
whereas nonparametric methods do not deliver two-sided bounds.
The flip side is that building accurate proposals for EEVI needs
more expertise as compared to nonparametric methods.

\citet{rainforth2018} give a thorough treatment of consistency and
convergence properties of a very general class of nested Monte Carlo
estimators.
The expressions in \cref{eq:Hl-generic,eq:Hu-generic} used for EEVI are
instances of nested Monte Carlo, where the inner expectation is obtained via
pseudo-marginal methods~\citep{andrieu2009} and the non-linear mapping
is log.

\citet{grosse2015} introduced the idea of using annealed importance sampling or
sequential harmonic mean to ``sandwich'' marginal log probabilities.
\citet{grosse2016} and~\citet{wu2017} applied these estimators to
diagnose MCMC inference algorithms and analyze deep generative models,
respectively.
Our work develops log probability bounds for the new problem of
squeezing entropy values and forming interval estimators, as well as
composing the estimates to bound many information-theoretic
quantities.
\citet{towner2017} use a similar family of auxiliary-variable
importance samplers from~\cref{sec:aux-var-smc} to upper bound
symmetric KL divergences between a pair of distributions---in that
setting, the normalizing constants
in~\cref{eq:is-log-ratio-biased} are irrelevant as they cancel out
so the weights~\cref{eq:w-extended-lo,eq:w-extended-hi} are only needed
up to normalizing constants.
In contrast, the normalizing constants in EEVI cannot be ignored as
they are the essential quantities needed to bound entropies.

We implemented the EEVI
(\cref{alg:entropy-upper-general,alg:entropy-lower-general}) as
meta-programs in the Gen probabilistic programming
system~\citep{towner2019gen}, available in the online code supplement.
These implementations make it easy to apply EEVI to a broad set of
generative models that are specified as probabilistic programs in Gen,
provided that the target random variables correspond to random choices
at addresses that exist in each execution of the program.
Moreover, EEVI is more widely applicable than previous
probabilistic programming-based estimators for information-theoretic quantities, such
as \citet[Alg.~2a]{saad2017dependencies}, \citet[Fig.~11]{gehr2020},
and~\citet[Sec.~8.2]{Narayanan2020}---these estimators assume that the probabilistic
programming system can exactly compute any marginal or conditional
density, which is rarely possible except in languages
that restrict modeling expressiveness to enable
exact inference~\citep{saad2016cgpm,gehr2016,narayanan2016,saad2021sppl}.

\section{CONCLUSION}
\label{sec:conclusion}

We have introduced estimators of entropy via inference (EEVI), a new
solution to the fundamental problem of estimating the entropy of
arbitrary variables in a generative model by leveraging probabilistic
inference.
EEVI computes interval estimates of entropy that can be
composed to accurately squeeze several other information-theoretic quantities.
The experiments show that EEVI delivers information-theoretically optimal
solutions to challenging problems from hepatology and endocrinology.
Avenues for future work include using EEVI for optimal experiment design and
information analysis in widely used medical expert
systems~\citep{Zhou2021} by developing the application in tandem with
clinical domain experts.
To increase the level of automation and accuracy of EEVI, it may be
worthwhile to leverage recent probabilistic programming methods that
automatically learn amortized proposal distributions~\citep{paige2016,ritchie2016,le2017}.
Another direction is using EEVI to extend the class of
models and queries that can be handled by existing probabilistic
programming-based frameworks for tasks such as searching structured
databases~\citep{saad2017search} and optimal experiment
design~\citep{ouyang2018} that need accurate estimates of entropy and information.

\bibliographystyle{plainnat}
\bibliography{paper}

\clearpage
\appendix


\thispagestyle{empty}

\twocolumn[ \makesupplementtitle ]

\input{figures/gaps}

\section{SPECIAL CASES OF EEVI}
\label{appx:eevi-variants}

\cref{alg:entropy-upper-general,alg:entropy-lower-general}
in the main text present a general version of Monte Carlo lower
and upper bounds for the entropy $H(Y) = -\E{\log p(Y)}$ using
the identities~\cref{eq:Hu-generic,eq:Hl-generic,eq:Hy-sandwich},
where the functions $w$ and $w'$ are as defined in~\cref{eq:w-extended-lo,eq:w-extended-hi}.
The general version of EEVI uses a proposal distribution $q(v, x; y)$
that is defined on an extended space $\mathcal{V} \times \mathcal{X}$
and parameterized by $\mathcal{Y}$, and the auxiliary proposals $r(v;
x,y)$ are defined on $\mathcal{V}$ and parameterized by $\mathcal{X}
\times \mathcal{Y}$.
This appendix rephrases the general version of EEVI in terms of bounds
on the negative entropy $-H(Y) = \E{\log p(Y)}$ and presents three
special cases described in~\cref{sec:aux-var-smc} of the main text.


\begin{enumerate}
\setlength{\intextsep}{0pt}
\setlength{\textfloatsep}{0pt}

\item \cref{alg:entropy-upper-general-appx,alg:entropy-lower-general-appx}
show Monte Carlo bounds on $\E{\log p(Y)}$ using a proposal
$q(v,x; y)$ and auxiliary proposal
$r(v;x,y)$ that have arbitrary auxiliary variables and are defined on
the extended state-spaces $\mathcal{V} \times \mathcal{X}$ and
$\mathcal{V}$, respectively.
The proposals used for the lower and upper bounds need not
be the same.
\Cref{fig:gaps} shows the estimation gap that results from using these
estimators of $\E{\log p(Y)}$ to estimate $H(Y) = -\E{\log p(Y)}$.

\item \cref{alg:entropy-lower-bound-basic,alg:entropy-upper-bound-basic}
  show Monte Carlo bounds for the case that the proposal $q$ has no auxiliary
  variables, so that $\mathcal{V} = \set{\omega}$ is a singleton and
  $r(v; x, y)$ is a dirac measure on $\omega$, as discussed in~\cref{sec:aux-var}.
  Proposals of this form may arise when $q$ is hand-constructed
  or learned via amortized variational inference.

\item \cref{alg:entropy-lower-bound-sir,alg:entropy-upper-bound-sir}
  show Monte Carlo bounds for the case that the proposal is sampling-importance
  resampling (SIR) with $P$ particles, using a base proposal $q_0(x;y)$ that has no
  auxiliary variables. The resulting proposal $q(v,x;y)$ and auxiliary
  proposal $r(v;x,y)$ are as in~\cref{example:sir} from the main text.

\item \cref{alg:entropy-lower-bound-sir-nested,alg:entropy-upper-bound-sir-nested}
  show Monte Carlo bounds using a base proposal $q_0(v, x;y)$ that has auxiliary
  variables with a corresponding base auxiliary proposal $r_0(v; x, y)$, and
  $P$ particles of SIR to estimate the marginal probability $q_0(x;y)$.
  For example, the base proposal $q_0$ and auxiliary proposal $r_0$ may
  follow the sequential Monte Carlo (SMC)
  scheme from~\cref{alg:smc-forward,alg:smc-reverse}, described in \cref{example:smc} of the main text.
  %
  %
  In this case, the overall proposal $q(v, x; y)$ and auxiliary proposal $r(v; x, y)$
  on the extended state space can be written as:
  \begin{align*}
    q(v_{1:P}, x; y) &= \frac{1}{P}\sum_{k=1}^P q_0(v_k, x; y) \prod_{\substack{t=1 \\ t\ne k}}^P r_0(v_t; x, y), \\
    r(v_{1:P}; x, y) &= \prod_{t=1}^P r_0(v_t; x, y).
  \end{align*}
  The extended weight~\cref{eq:w-extended-lo} from the main text,
  which appears in~\cref{alg:entropy-lower-bound-sir-nested} \cref{algline:entropy-lower-bound-sir-nested-weight},
  is then
  \begin{align}
  &w(v_{1:P}, x; y) \notag \\
  &= \frac
    {\displaystyle p(x,y)\prod_{k=1}^P r_0(v_k; x, y)}
    {\displaystyle\frac{1}{P}\sum_{k=1}^P q_0(v_k, x; y) \prod_{\substack{t=1 \\ t\ne k}}^P r_0(v_t; x, y)} \notag \\
  &= \frac{p(x,y)}{\displaystyle\frac{1}{P}\sum_{k=1}^P\frac{q_0(v_k,x;y)}{r_0(v_k;x,y)}} \label{eq:melanose}.
  \end{align}
  The extended weight~\cref{eq:w-extended-hi} from the main text,
  which appears in~\cref{alg:entropy-upper-bound-sir-nested} \cref{algline:entropy-upper-bound-sir-nested-weight},
  is the reciprocal of~\cref{eq:melanose}.
  The extended proposal $q(v_{1:P}, x; y)$ generates
  samples $(V_{1:P}, X)$
  as follows: \begin{itemize}[wide=0pt]
    \item sample $(V_0, X) \sim q_0(v, x; y)$;
    \item sample selection index $k \sim \mathrm{Uniform}(1 \dots P)$;
    \item set $V_k \gets V_0$;
    \item sample $V_j \sim r_0(v; X, y)$ from the base auxiliary proposal
    for $j = 1, \dots, k - 1, k+1, \dots, P$.
  \end{itemize}
  As for the extended auxiliary proposal $r(v_{1:P}; x, y)$,
  it generates $P$ i.i.d.\ samples from the base auxiliary
  proposal $r_0(v; x, y)$.
\end{enumerate}

\begin{figure*}[p]
\begin{minipage}[t]{.475\linewidth}
\begin{algorithm}[H]
\caption{Lower bound on $\E{\log p(Y)}$}
\label{alg:entropy-upper-general-appx}
\begin{algorithmic}[1]
\Require \begin{tabular}{l}
  Target distribution $p(x,y)$ \\
  Proposal distribution $q(v, x; y)$ \\
  Auxiliary proposal distribution $r(v; x, y)$ \\
  Number of samples $n$, $m$
  \end{tabular}
\For{$i=1\dots n$}
  \State $(\tilde{X}, Y) \sim p(x,y)$
  \For{$j=1\dots m$}
    \State $(V, X) \sim q(v,x; Y)$
    \State $\displaystyle t'_j \gets \log \frac{p(X, Y)r(V; X, Y)}{q(V, X; y)}$
  \EndFor
\State $t_i \gets \frac{1}{m} \sum_{j=1}^m t'_j$
\EndFor
\State \Return $\frac{1}{n} \sum_{i=1}^n t_i$
\end{algorithmic}
\end{algorithm}
\end{minipage}\hfill
\begin{minipage}[t]{.475\linewidth}
\begin{algorithm}[H]
\caption{Upper bound on $\E{\log p(Y)}$}
\label{alg:entropy-lower-general-appx}
\begin{algorithmic}[1]
\Require \begin{tabular}{l}
  Target distribution $p(x,y)$ \\
  Proposal distribution $q(v, x; y)$ \\
  Auxiliary proposal distribution $r(v; x, y)$ \\
  Number of samples $n$, $m$
  \end{tabular}
\For{$i=1\dots n$}
  \State $(X_1, Y) \sim p(x, y)$
  \State $(X_{2:m}) \sim \mathrm{MCMC}_{X_1} \mbox{ targeting } p(x \mid Y)$
  \For{$j=1 \dots m$}
    \State $V \sim r(v; X_j, Y)$
    \State $\displaystyle t'_j \gets \log \frac{q(V, X_j; Y)}{p(X_j, Y)r(V; X_j, Y)}$
  \EndFor
  \State $ t_i \gets - \frac{1}{m} \sum_{j=1}^m t'_j$
\EndFor
\State \Return $\frac{1}{n} \sum_{i=1}^n t_i$
\end{algorithmic}
\end{algorithm}
\end{minipage}

\begin{minipage}[t]{.475\linewidth}
\begin{algorithm}[H]
\captionsetup{skip=100pt}
\caption{Lower bound on $\E{\log p(Y)}$ using a proposal distribution
without auxiliary variables}
\label{alg:entropy-lower-bound-basic}
\begin{algorithmic}[1]
\Require \begin{tabular}{l}
  Target distribution $p(x, y)$ \\
  Proposal distribution $q(x; y)$ \\
  Number of samples $n$, $m$
  \end{tabular}
\For{$i=1\dots n$}
  \State $Y \sim p(x, y)$
  \For{$j=1\dots m$}
    \State $X\sim q(x; Y)$
    \State $\displaystyle t'_j \gets \log \frac{p(X, Y)}{q(X; y)}$
  \EndFor
\State $t_i \gets \frac{1}{m} \sum_{j=1}^m t'_j$
\EndFor
\State \Return $\frac{1}{n} \sum_{i=1}^n t_i$
\end{algorithmic}
\end{algorithm}
\end{minipage}\hfill
\begin{minipage}[t]{.475\linewidth}
\begin{algorithm}[H]
\caption{Upper bound on $\E{\log p(Y)}$ using a proposal distribution
without auxiliary variables}
\label{alg:entropy-upper-bound-basic}
\begin{algorithmic}[1]
\Require \begin{tabular}{l}
  Target distribution $p(x, y)$ \\
  Proposal distribution $q(x; y)$ \\
  Number of samples $n$, $m$
  \end{tabular}
\For{$i=1\dots n$}
  \State $(X_1, Y) \sim p(x, y)$
  \State $(X_{2:m}) \sim \mathrm{MCMC}_{X_1} \mbox{ targeting } p(x \mid Y)$
  \For{$j=1 \dots m$}
    \State $\displaystyle t'_j \gets \log \frac{q(X_j; Y)}{p(X_j, Y)}$
  \EndFor
  \State $ t_i \gets - \frac{1}{m} \sum_{j=1}^m t'_j$
\EndFor
\State \Return $\frac{1}{n} \sum_{i=1}^n t_i$
\end{algorithmic}
\end{algorithm}
\end{minipage}
\end{figure*}

\begin{figure*}[p]
\begin{minipage}[t]{.475\linewidth}
\begin{algorithm}[H]
\caption{Lower bound on $\E{\log p(Y)}$
using SIR and a base proposal without auxiliary variables}
\label{alg:entropy-lower-bound-sir}
\begin{algorithmic}[1]
\Require \begin{tabular}{l}
  Target distribution $p(x,y)$ \\
  Base proposal distribution $q_0(x; y)$ \\
  Number of samples $n$, $m$, $P$
  \end{tabular}
\For{$i=1\dots n$}
  \State $(\tilde{X}, Y) \sim p(x, y)$
  \For{$j=1\dots m$}
    \For{$k=1\dots P$}
      \State $X\sim q_0(x; Y)$
      \State $\displaystyle \xi_k \gets \frac{p(X, Y)}{q_0(X; Y)}$
    \EndFor
    \State $ t'_j \gets \log \left[\frac{1}{P}\sum_{k=1}^P \xi_k \right]$
  \EndFor
\State $t_i \gets \frac{1}{m} \sum_{j=1}^m t'_j$
\EndFor
\State \Return $\frac{1}{n} \sum_{i=1}^n t_i$
\end{algorithmic}
\end{algorithm}
\end{minipage}\hfill
\begin{minipage}[t]{.475\linewidth}
\begin{algorithm}[H]
\caption{Upper bound on $\E{\log p(Y)}$
using SIR and a base proposal without auxiliary variables}
\label{alg:entropy-upper-bound-sir}
\begin{algorithmic}[1]
\Require \begin{tabular}{l}
  Target distribution $p(x,y)$ \\
  Base proposal distribution $q_0(x; y)$ \\
  Number of samples $n$, $m$, $P$
  \end{tabular}
\For{$i=1\dots n$}
  \State $(X_1, Y) \sim p(x, y)$
  \State $(X_{2:m}) \sim \mathrm{MCMC}_{X_1} \mbox{ targeting } p(x \mid Y)$
  \For{$j=1 \dots m$}
    \State $X'_1 \gets X_j$
    \For{$k =2 \dots P$}
      \State $X'_k \sim q_0(x; Y)$
    \EndFor
    \For{$k = 1 \dots P$}
      \State $\displaystyle \xi_k \gets \frac{p(X'_k, Y)}{q_0(X'_k; y)}$
    \EndFor
    \State $t'_j \gets \log\left[ \textstyle {1}/{\frac{1}{P}\sum_{k=1}^P \xi_k} \right]$
  \EndFor
  \State $ t_i \gets - \frac{1}{m} \sum_{j=1}^m t'_j$
\EndFor
\State \Return $ \frac{1}{n} \sum_{i=1}^n t_i$
\end{algorithmic}
\end{algorithm}
\end{minipage}

\begin{minipage}[t]{.475\linewidth}
\begin{algorithm}[H]
\caption{Lower bound on $\E{\log p(Y)}$
using SIR and a base proposal with auxiliary variables}
\label{alg:entropy-lower-bound-sir-nested}
\begin{algorithmic}[1]
\Require \begin{tabular}{l}
  Target distribution $p(x,y)$ \\
  Base proposal distribution $q_0(v, x; y)$ \\
  Base auxiliary proposal dist $r_0(v; x, y)$ \\
  Number of samples $n$, $m$, $P$
  \end{tabular}
\For{$i=1\dots n$}
  \State $(\tilde{X}, Y) \sim p(x, y)$
  \For{$j=1\dots m$}
    \State $(V_1, X) \sim q_0(v, x; Y)$
    \For{$k=2\dots P$}
      \State $V_k \sim r_0(v; X, Y)$
    \EndFor
    \For{$k=1\dots P$}
      \State $\xi_k \gets \displaystyle \frac{q_0(V_k, X; Y)}{r_0(V_k; X, Y)}$
    \EndFor
    \State $\displaystyle t'_j \gets \log \frac{p(X, Y)}{ \frac{1}{P}\sum_{k=1}^P \xi_k}$
    \label{algline:entropy-lower-bound-sir-nested-weight}
  \EndFor
\State $t_i \gets \frac{1}{m} \sum_{j=1}^m t'_j$
\EndFor
\State \Return $\frac{1}{n} \sum_{i=1}^n t_i$
\end{algorithmic}
\end{algorithm}
\end{minipage}\hfill
\begin{minipage}[t]{.475\linewidth}
\begin{algorithm}[H]
\caption{Upper bound on $\E{\log p(Y)}$
using SIR and a base proposal with auxiliary variables}
\label{alg:entropy-upper-bound-sir-nested}
\begin{algorithmic}[1]
\Require \begin{tabular}{l}
  Target distribution $p(x,y)$ \\
  Base proposal distribution $q_0(v, x; y)$ \\
  Base auxiliary proposal dist $r_0(v; x, y)$ \\
  Number of samples $n$, $m$, $P$
  \end{tabular}
\For{$i=1\dots n$}
  \State $(X_1, Y) \sim p(x, y)$
  \State $(X_{2:m}) \sim \mathrm{MCMC}_{X_1} \mbox{ targeting } p(x \mid Y)$
  \For{$j=1\dots m$}
    \For{$k=1\dots P$}
      \State $V_k \sim r_0(v; X_j, Y)$
      \State $\xi_k \gets \displaystyle \frac{q_0(V_{k}, X_{j}; Y)}{r_0(V_k; X_j, Y)}$
    \EndFor
    \State $\displaystyle t'_j \gets \log \frac{ \frac{1}{P}\sum_{k=1}^P \xi_k}{p(X_{j}, Y)}$
    \label{algline:entropy-upper-bound-sir-nested-weight}
  \EndFor
\State $ t_i \gets - \frac{1}{m} \sum_{j=1}^m t'_j$
\EndFor
\State \Return $\frac{1}{N} \sum_{i=1}^N t_i$
\end{algorithmic}
\end{algorithm}
\end{minipage}
\end{figure*}

\section{DEFERRED PROOFS}
\label{appx:proofs}

This appendix establishes~\crefrange{eq:is-ratio-unbisaed}{eq:is-log-ratio-ead}
from~\cref{sec:aux-var} in the main text using elementary arguments.
The basic setup is as follows: let $(\mathcal{X}, \mathcal{B}(\mathcal{X}))$
be a measurable space and
suppose that $G(\diff x)$ is a probability measure
that admits a density $g(x)$
with respect to some $\sigma$-finite measure $\nu(\diff{x})$.
Therefore,
\begin{align}
G(A) = \int_{\mathcal{X}} \mathbb{I}[x \in A] g(x) \nu(\diff x)
&& (A \in \mathcal{B}(\mathcal{X})),
\end{align}
and for any bounded measurable function $\phi: \mathcal{X} \to \mathbb{R}$,
\begin{align}
\int_{\mathcal{X}}\phi(x) G(\diff x) = \int_{\mathcal{X}} \phi(x) g(x) \nu(\diff x).
\end{align}
Assume now that $g(x) = \tilde{g}(x)/Z_g$ is known only up to a
normalizing constant

\begin{proposition}
\label{prop:Zg-norm}
The normalizing constant $Z_g$ of $g$
satisfies $Z_g = \int_{\mathcal{X}}\tilde{g}(x)\nu(\diff{x})$.
\end{proposition}
\begin{proof}
\begin{align}
1 = G(\mathcal{X})
  &= \int_{\mathcal{X}} g(x) \nu(\diff x) \\
  &= \int_{\mathcal{X}} [{\tilde{g}(x)}/{Z_g}] \nu(\diff x) \\
  &= \int_{\mathcal{X}} [{\tilde{g}(x)}/{Z_g}] \nu(\diff x).
\end{align}
\end{proof}

\begin{proposition}
\label{prop:G-measure-zero}
Let $\mathcal{G} \defas \set{x \mid g(x) = 0}$ be the set of values for which
$g$ is zero and let $G' \defas \mathcal{X} \setminus \mathcal{G}$
be its complement.
Then $G(\mathcal{G}) = 0$.
\end{proposition}
\begin{proof}
As $g$ is a measurable function, the set $\mathcal{G}$ is also
measurable and has $G$-probability
\begin{align}
G(\mathcal{G})
  = \int_{\mathcal{X}}\mathbb{I}[x \in \mathcal{G}]g(x) \nu(\diff{x}) 
  &= \int_{\mathcal{G}}g(x) \nu(\diff{x})\\
  &= \int_{\mathcal{G}}0 \nu(\diff{x})\\
  &= 0.
\end{align}
\end{proof}

\begin{corollary}
\label{corr:phi-arbitrary}
For measurable function $\phi: \mathcal{X} \to \mathbb{R}$, we have
\begin{align}
\int_{\mathcal{X}}\phi(x)G(\diff{x})
  = \int_{\mathcal{G}'}\phi(x)G(\diff{x}).
\end{align}
\end{corollary}

Let $H$ be a probability measure that is absolutely continuous with respect
to $\nu$ with density $h(x) = \tilde{h}(x)/Z_h$.
\cref{prop:Zg-norm} implies that
$Z_h = \int_{\mathcal{X}}\tilde{h}(x)\nu\diff{x}$.
Further, let
\begin{align}
\label{eq:w-unnormalized-ratio}
\tilde{w}(x) &\defas \begin{cases}
  \displaystyle\frac{\tilde{h}(x)}{\tilde{g}(x)} & \mbox{if } x \in \mathcal{G}', \\
  1 & \mbox{if } x \in \mathcal{G},
\end{cases}\\
w(x) &\defas \frac{Z_h}{Z_g}\tilde{w}(x) = \begin{cases}
  \displaystyle\frac{h(x)}{g(x)} & \mbox{if } x \in \mathcal{G}', \\
  1 & \mbox{if } x \in \mathcal{G}.
\end{cases}
\end{align}

\begingroup
\allowdisplaybreaks
\begin{proposition}
\label{prop:w-density}
If $H \ll G$ then $w(x)$ is a density of $H$ with respect to $G$.
\end{proposition}
\begin{proof}
Since $H \ll G$ and $G(\mathcal{G}) = 0$ it also holds that $H(\mathcal{G}) = 0$.
Thus~\cref{corr:phi-arbitrary} also applies when taking expectations
under $H$.
For any $A \in \mathcal{B}(\mathcal{X})$
\begin{align}
&\int_{\mathcal{X}}\mathbb{I}[x \in A]w(x)G(\diff{x}) \\
  &= \int_{\mathcal{G}'}\mathbb{I}[x \in A]w(x)G(\diff{x}) \\
  &= \int_{\mathcal{G}'}\mathbb{I}[x \in A][h(x)/g(x)]G(\diff{x}) \\
  &= \int_{\mathcal{G}'}\mathbb{I}[x \in A][h(x)/g(x)]g(x)\nu(\diff{x}) \\
  &= \int_{\mathcal{G}'}\mathbb{I}[x \in A]h(x)\nu(\diff{x}) \\
  &= H(A).
\end{align}
\end{proof}
\endgroup

\begin{corollary}
\label{corr:expected-w-unity}
If $H \ll G$ and $X \sim G$, then $\E{w(X)} = 1$.
\end{corollary}

\Cref{prop:is-ratio-unbisaed} establishes~\cref{eq:is-ratio-unbisaed} from
the main text.
\begin{proposition}
\label{prop:is-ratio-unbisaed}
If $H \ll G$ and $X \sim G$, then $\mathbb{E}[\tilde{w}(X)] = Z_h/Z_g$.
\end{proposition}
\begin{proof}
\Cref{eq:w-unnormalized-ratio} implies that
$\tilde{w}(x) = [Z_h/Z_g]w(x)$ for all $x \in \mathcal{X}$.
Applying~\cref{corr:expected-w-unity}:
\begin{align}
\E{\tilde{w}(X)}
  &= \E{[{Z_h}/{Z_g}]w(X)} \\
  &= [Z_h/Z_g]\E{w(X)} \\
  &= Z_h/Z_g.
\end{align}
\end{proof}

\Cref{prop:w-as-positive,prop:is-log-ratio-biased}
establish~\cref{eq:is-log-ratio-biased} from
the main text.

\begin{proposition}
\label{prop:w-as-positive}
If $G \ll H$ then $\tilde{w}(x)$ is $G$-almost surely positive.
\end{proposition}
\begin{proof}
Suppose towards a contradiction that there exists a measurable set $A$ such
that $G(A) > 0$ and $\tilde{w}(x) = 0$ for $x \in A$.
From~\cref{eq:w-unnormalized-ratio}, it must be that $h(x) = 0$ whenever $x \in A$.
But then $H(A) = \int_{A}h(x) H(\diff{x}) = 0$, a contradiction to $G \ll H$.
\end{proof}

\begin{proposition}
\label{prop:is-log-ratio-biased}
If $G \ll H$ and $X \sim G$, then
  $\E{\log \tilde{w}(X)} = \log(Z_h/Z_g) - \KL{G}{H}$.
\end{proposition}
\begin{proof}
\Cref{prop:w-as-positive} and
the definition $w(x)$ in~\cref{eq:w-unnormalized-ratio}
together imply that $\KL{G}{H} = - \E{\log w(X)}$ for $X \sim G$.
%
%
Since $G \ll H$, we have
\begin{align}
\E{\log \tilde{w}(X)}
  &= \E{\log \frac{Z_h}{Z_g} w(X)} \\
  &= \log(Z_h/Z_g) + \E{\log w(X)} \\
  &= \log(Z_h/Z_g) - \KL{G}{H}.
\end{align}
\end{proof}

\begin{proposition}
\label{prop:is-log-ratio-var}
If $G \ll H$ and $X \sim G$, then
$\mathrm{Var}\left[\log \tilde{w}(X) \right]
  = \E{ \log^2 w(X) } - \left(\KL{G}{H}\right)^2$.
\end{proposition}
\begin{proof}
Recall that for any real random variable $B$ and constant $c$,
$\mathrm{Var}[B] = \E{B^2} - (\E{B})^2$ and
$\mathrm{Var}[B+c] = \mathrm{Var}[B]$.
Applying this property to the random variable $\log \tilde{w}(X)$ gives
\begin{align}
&\mathrm{Var}\left[\log \tilde{w}(X) \right] \\
  &= \mathrm{Var}\left[\log \frac{Z_h}{Z_g} + \log w(X) \right] \\
  &= \mathrm{Var}\left[\log w(X) \right] \\
  &= \E{\log^2 w(X)} - \left(\E{\log w(X)} \right)^2 \\
  &= \E{\log^2 w(X)} - \left(-\KL{G}{H} \right)^2 \\
  &= \E{\log^2 w(X)} - \left(\KL{G}{H} \right)^2.
\end{align}
\end{proof}

\Cref{prop:is-log-ratio-var} establishes~\cref{eq:is-log-ratio-var}.
We next establish~\cref{eq:is-log-ratio-tail}.
\begin{proposition}
\label{prop:is-log-ratio-tail}
If $H \ll G$ and $X \sim G$, then
\begin{align}
\Pr\left[\log \tilde{w}(X) \ge t + \log(Z_h/Z_g)\right] \le e^{-t}
\end{align}
for any $t > 0$.
\end{proposition}
\begin{proof}
\Cref{prop:is-ratio-unbisaed} gives $\E{\tilde{w}(X)} = Z_h/Z_g$.
Further, from~\cref{prop:w-as-positive} $\tilde{w}$ is $G$-almost surely positive.
Thus,
$\Pr\left[\tilde{w}(X) \ge e^t(Z_h/Z_g)\right] \le e^{-t}$
for $t > 0$.
Applying $\log$ to both sides as in~\cref{eq:is-log-ratio-tail}
completes the proof.
\end{proof}

\allowdisplaybreaks

\Cref{eq:is-log-ratio-ead} in the main text is a special case
of a result in~\citet[Appx.~B]{burda2016} to importance sampling
estimators that satisfy~\cref{prop:is-log-ratio-biased,prop:is-log-ratio-tail}.
For completeness, we give the proof in two stages using the notation in this paper below.

\begin{proposition}
Let $B$ be a real random variable with finite expectation and $\mu \defas \E{B}$
denote its expectation.
Then $\E{\lvert {B - \mu} \rvert} = 2\E{\max(0, B-\mu)}$.
\end{proposition}

\begin{proof}
By additivity of expectation
\begin{align}
\E{|B-\mu|}
=\begin{aligned}[t]
  & \E{|B-\mu| \cdot \mathbf{1}_{B > \mu}} \\
  &+ \E{|B-\mu| \cdot \mathbf{1}_{B < \mu}} \\
  & + \E{|B-\mu| \cdot \mathbf{1}_{B = \mu}}.
  \label{eq:addivity-expectation}
\end{aligned}
\end{align}
As the third term in the right hand side
of~\cref{eq:addivity-expectation} is zero, it suffices to establish
that the first two terms on the right hand side are equal:
\begin{align}
\mu &= \E{B\mathbf{1}_{B > \mu}} + \E{B\mathbf{1}_{B < \mu}} + \E{B\mathbf{1}_{B = \mu}} \\
\mu &= \E{B\mathbf{1}_{B > \mu}} \begin{aligned}[t]
  &+ \E{B\mathbf{1}_{B < \mu}} \\
  &+ \mu \E{(1-(\mathbf{1}_{B > \mu} + \mathbf{1}_{B < \mu}))}
  \end{aligned} \\
\mu &= \E{B\mathbf{1}_{B > \mu}} \begin{aligned}[t]
  &+ \E{B\mathbf{1}_{B < \mu}} + \mu \\
  & - \mu \E{\mathbf{1}_{B > \mu}} - \mu \E{\mathbf{1}_{B < \mu}}
  \end{aligned} \\
0 &= \E{(B -\mu)\mathbf{1}_{B > \mu}} - \E{(\mu-B)\mathbf{1}_{B < \mu}}.
\end{align}
Using the fact that $-X\mathbf{1}_{X < 0} = |X|\mathbf{1}_{X < 0}$
and $X\mathbf{1}_{X > 0} = |X|\mathbf{1}_{X > 0}$ a.s., for any random variable $X$,
we have
\begin{align}
\E{(\mu- B)\mathbf{1}_{B < \mu}} &= \E{(B -\mu)\mathbf{1}_{B > \mu}} \label{eq:additivity-expectation-pre-final} \\
\E{|\mu- B|\cdot\mathbf{1}_{B < \mu}} &= \E{|B -\mu|\cdot\mathbf{1}_{B > \mu}}. \label{eq:additivity-expectation-final}
\end{align}

Combining~\cref{eq:addivity-expectation,eq:additivity-expectation-pre-final,eq:additivity-expectation-final} gives.
\begin{align}
\E{|B-\mu|}
  &= 2\E{|B-\mu|\cdot \mathbf{1}_{B > \mu}} \\
  &= 2\E{(B-\mu)\cdot \mathbf{1}_{B > \mu}} \\
  &= 2\E{\max(0, B-\mu)}.
\end{align}
\end{proof}

\begin{proposition}
\label{prop:is-log-ratio-ead}
If $H$ and $G$ are mutually absolutely continuous, i.e., $G \ll H$, $H \ll G$,
and $X \sim G$, then
\begin{align}
\E{\lvert \log \tilde w(X) - \E{ \log \tilde w(X)} \rvert} \le 2 + 2\KL{G}{H}.
\end{align}
\end{proposition}
\begin{proof}
For any real random variable $B$, let $(B)_{+} \defas \max(0, B)$ denote
the positive part.
Then
\begin{align}
&\E{\lvert \log \tilde w(X) - \E{ \log \tilde w(X)} \rvert} \\
  &= 2\E{\left( \log \tilde{w}(X) - \E{ \log \tilde w(X)} \right)_{+}} \\
  &= \begin{aligned}[t]
    2\mathbb{E} \Big[ &\big( \log \tilde{w}(X) - \log(Z_h/Z_g)  \\
      &+ \log(Z_h/Z_g) - \E{ \log \tilde w(X)} \big)_{+} \Big]
    \end{aligned} \\
  &\le 2\mathbb{E} \begin{aligned}[t]\big[
    &\left( \log \tilde{w}(X) - \log(Z_h/Z_g)\right)_{+} \\
    &+ \left(\log(Z_h/Z_g) - \E{ \log \tilde w(X)} \right)_{+}
    \big]
    \end{aligned} \\
  &= \begin{aligned}[t]
    &2\E{\left( \log \tilde{w}(X) - \log(Z_h/Z_g)\right)_{+}} \\
    &+ 2\left(\log(Z_h/Z_g) - \E{ \log \tilde w(X)} \right)_{+}
    \end{aligned} \\
  &= \begin{aligned}[t]
    &2\E{\left( \log \tilde{w}(X) - \log(Z_h/Z_g)\right)_{+}} \\
    &+ 2\KL{G}{H}
    \end{aligned} \\
  &= \begin{aligned}[t]
    & 2\int_{0}^{\infty} \Pr[ \log \tilde{w}(X) - \log(Z_h/Z_g) > t ]\diff{t}\\
    & + 2\KL{G}{H}
    \end{aligned} \\
  &= \begin{aligned}[t]
    &2\int_{0}^{\infty} \Pr[ \log \tilde{w}(X) >  \log(Z_h/Z_g) + t ]\diff{t}\\
    &+ 2\KL{G}{H}
    \end{aligned} \\
  &\le 2\int_{0}^{\infty} \exp(-t)\diff{t}  + 2\KL{G}{H} \\
  &= 2 + 2\KL{G}{H}].
\end{align}
\end{proof}

\end{document}

%% file: figures/derived.tex
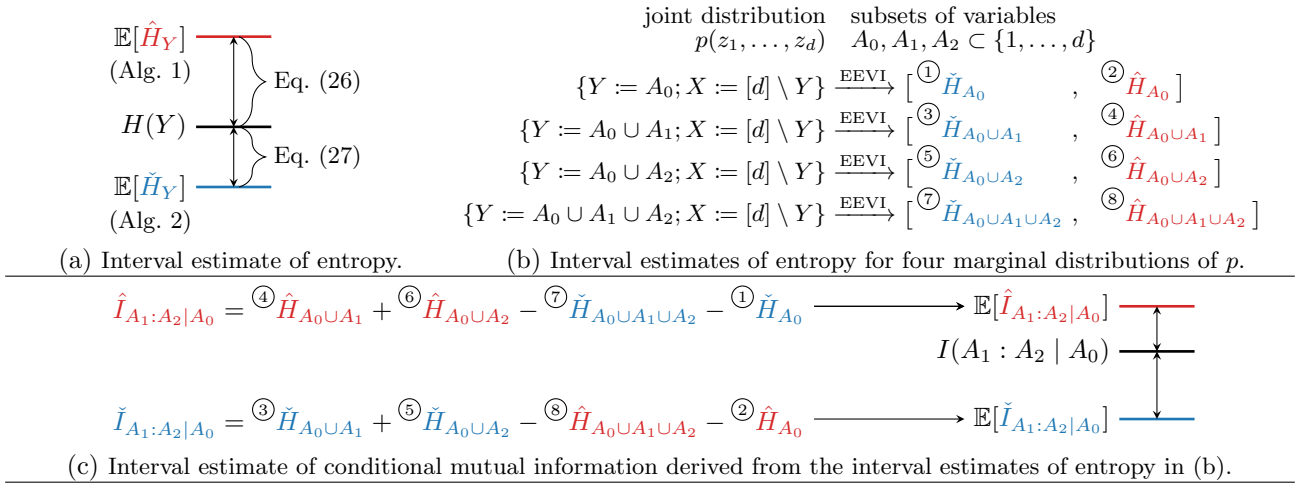
\begin{figure*}
\FrameSep0pt
\hrule
\definecolor{mplBlue}{HTML}{1F77B4}
\definecolor{mplRed}{HTML}{D62728}
\newcommand{\LB}[1]{\textcolor{mplBlue}{#1}}
\newcommand{\UB}[1]{\textcolor{mplRed}{#1}}
\newcommand{\tiksuper}[1]{\tikz{\node[draw=black,circle,inner sep=1pt,outer sep=0pt,font=\tiny](){#1};}}
\begin{subfigure}[b]{.35\linewidth}
\centering
\begin{tikzpicture}
\node[name=xtl, coordinate] at (0,0) {};
\node[name=xtr, coordinate, right=1 of xtl] {};
\node[name=xbl, coordinate, below = 2 of xtl] {};
\node[name=xbr, coordinate] at (xbl -| xtr) {};

\node[name=xml, coordinate] at ($(xtl)!.6!(xbl)$) {};
\node[name=xmr, coordinate] at (xml -| xtr) {};

\draw[line width=1pt, color=mplRed] (xtl) -- (xtr);
\draw[line width=1pt, color=mplBlue] (xbl) -- (xbr);
\draw[line width=1pt, color=black] (xml) -- (xmr);

\node[color=black, left=0 of xtl, anchor=east,
  label={[font=\footnotesize,yshift=4pt,anchor=north]below:{(\cref{alg:entropy-upper-general})}}]
  {$\mathbb{E}[\UB{\hat{H}_Y}]$};

\node[color=black, left=0 of xbl, anchor=east,
  label={[font=\footnotesize,yshift=4pt,anchor=north]below:{(\cref{alg:entropy-lower-general})}}]
  {$\mathbb{E}[\LB{\check{H}_Y}]$};

\node[color=black, left=0 of xml, anchor=east]
  {$H(Y)$};

\draw[stealth-stealth] ($(xml)!.5!(xmr)$) -- ($(xtl)!.5!(xtr)$);
\draw[stealth-stealth] ($(xml)!.5!(xmr)$) -- ($(xbl)!.5!(xbr)$);

\draw[decorate, decoration={brace,amplitude=10pt,mirror,raise=2pt}]
  ($(xml)!.5!(xmr)$) -- ($(xtl)!.5!(xtr)$)
  node[pos=.5, anchor=west,xshift=15pt, draw=none, inner sep=0pt] {\footnotesize \Cref{eq:gap-aux-upper}};
\draw[decorate, decoration={brace,amplitude=10pt,mirror,raise=2pt}]
  ($(xbl)!.5!(xbr)$) -- ($(xml)!.5!(xmr)$)
  node[pos=.5, anchor=west,xshift=15pt, draw=none, inner sep=0pt] {\footnotesize \Cref{eq:gap-aux-lower}};
\end{tikzpicture}
\captionsetup{skip=2pt}
\caption{Interval estimate of entropy.}
\end{subfigure}\hfill
\begin{subfigure}[b]{.65\linewidth}
\footnotesize
\setlength{\abovedisplayskip}{0pt}
\setlength{\belowdisplayskip}{0pt}
\setlength{\abovedisplayshortskip}{0pt}
\setlength{\belowdisplayshortskip}{0pt}
\begin{align*}
\textrm{\footnotesize joint distribution} &\quad \textrm{\footnotesize subsets of variables} \span\span\span\span\\[-4pt]
p(z_1, \dots, z_d) &\quad  A_0, A_1,A_2 \subset \set{1,\dots,d} \span\span\span\span\\[-2pt]
\left\lbrace Y \defas A_0; X \defas [d] \setminus Y \right\rbrace
  &\xrightarrow{\text{EEVI}}
  \big[\, ^{\tiksuper{1}}\LB{\check{H}_{A_0}} \hspace{-.25cm}&&,\hspace{.25cm} ^{\tiksuper{2}}\UB{\hat{H}_{A_0}}\,\big]
  \\[-4pt]
\left\lbrace Y \defas A_0 \cup A_1; X \defas [d] \setminus Y \right\rbrace
  &\xrightarrow{\text{EEVI}}
  \big[\, ^{\tiksuper{3}}\LB{\check{H}_{A_0 \cup A_1}} \hspace{-.25cm}&&,\hspace{.25cm} ^{\tiksuper{4}}\UB{\hat{H}_{A_0 \cup A_1}} \,\big]
  \\[-4pt]
\left\lbrace Y \defas A_0 \cup A_2; X \defas [d] \setminus Y \right\rbrace
  &\xrightarrow{\text{EEVI}}
  \big[\, ^{\tiksuper{5}}\LB{\check{H}_{A_0 \cup A_2}} \hspace{-.25cm}&&,\hspace{.25cm} ^{\tiksuper{6}}\UB{\hat{H}_{A_0 \cup A_2}} \,\big]
  \\[-4pt]
\left\lbrace Y \defas A_0 \cup A_1 \cup A_2; X \defas [d] \setminus Y \right\rbrace
  &\xrightarrow{\text{EEVI}}
  \big[\, ^{\tiksuper{7}}\LB{\check{H}_{A_0 \cup A_1 \cup A_2}} \hspace{-.25cm}&&,\hspace{.25cm} ^{\tiksuper{8}}\UB{\hat{H}_{A_0 \cup A_1 \cup A_2}} \,\big]
\end{align*}
\captionsetup{skip=0pt}
\caption{Interval estimates of entropy for four marginal distributions of $p$.}
\label{fig:derived-marginals}
\end{subfigure}
\hrule
\begin{subfigure}{\linewidth}
\centering
\begin{tikzpicture}
\node[name=xtl, coordinate] at (0,0) {};
\node[name=xtr, coordinate, right=1 of xtl] {};
\node[name=xbl, coordinate, below = 1.5 of xtl] {};
\node[name=xbr, coordinate] at (xbl -| xtr) {};

\node[name=xml, coordinate] at ($(xtl)!.4!(xbl)$) {};
\node[name=xmr, coordinate] at (xml -| xtr) {};

\draw[line width=1pt, color=mplRed] (xtl) -- (xtr);
\draw[line width=1pt, color=mplBlue] (xbl) -- (xbr);
\draw[line width=1pt, color=black] (xml) -- (xmr);

\draw[stealth-stealth] ($(xml)!.5!(xmr)$) -- ($(xtl)!.5!(xtr)$);
\draw[stealth-stealth] ($(xml)!.5!(xmr)$) -- ($(xbl)!.5!(xbr)$);
\node[name=IUB, color=black, left=0 of xtl, anchor=east]{$\mathbb{E}[\UB{\hat{I}_{A_1:A_2 \mid A_0}}]$};
\node[name=IEx,color=black, left=0 of xml, anchor=east]{$I(A_1: A_2 \mid A_0)$};
\node[name=ILB, color=black, left=0 of xbl, anchor=east]{$\mathbb{E}[\LB{\check{I}_{A_1:A_2 \mid A_0}}]$};

\node[name=IUBeq, color=black, left=2 of IUB] {
$\UB{\hat{I}_{A_1:A_2 \mid A_0}}
  = {^{\tiksuper{4}}\UB{\hat{H}_{A_0 \cup A_1}}}
  + {^{\tiksuper{6}}\UB{\hat{H}_{A_0 \cup A_2}}}
  - {^{\tiksuper{7}}\LB{\check{H}_{A_0 \cup A_1 \cup A_2}}}
  - {^{\tiksuper{1}}\LB{\check{H}_{A_0}}}$};
\draw[-stealth] (IUBeq) -- (IUB);

\node[name=ILBeq, color=black, left=2 of ILB] {
$\LB{\check{I}_{A_1:A_2\mid A_0}}
  = {^{\tiksuper{3}}\LB{\check{H}_{A_0 \cup A_1}}}
  + {^{\tiksuper{5}}\LB{\check{H}_{A_0 \cup A_2}}}
  - {^{\tiksuper{8}}\UB{\hat{H}_{A_0 \cup A_1 \cup A_2}}}
  - {^{\tiksuper{2}}\UB{\hat{H}_{A_0}}}$};

\draw[-stealth] (IUBeq) -- (IUB);
\draw[-stealth] (ILBeq) -- (ILB);
\end{tikzpicture}
\captionsetup{skip=2pt}
\caption{Interval estimate of conditional mutual information derived from the interval estimates of entropy in~\subref{fig:derived-marginals}.}
\hrule
\label{fig:derived-cmi}
\end{subfigure}
\captionsetup{skip=5pt, belowskip=-5pt}
\caption{Composing interval estimators of entropy to obtain bounds on derived information measures.}
\label{fig:derived}
\end{figure*}

%% file: figures/alg-eevi.tex
%
\begin{figure}[t]
\setlength{\intextsep}{0pt}

\centering
\def\colory{none}
\def\colorx{none}
\def\colorv{none}
\captionsetup[subfigure]{font=normal}
\begin{subfigure}[b]{.3\linewidth}
\centering
\begin{tikzpicture}
\node[draw=black, fill=\colory, circle] (y) {$y$};
\node[draw=black, fill=\colorx, circle, left= .5 of y] (x) {$x$};
\draw[-] (x) -- (y);
\end{tikzpicture}
\caption{\begin{tabular}[t]{@{}c}
  Target \\ $p(x, y)$
\end{tabular}}
\label{fig:distributions-aux-target}
\end{subfigure}
\begin{subfigure}[b]{.3\linewidth}
\centering
\begin{tikzpicture}
\node[name=y, draw=black, fill=black, label={right:$y$}] {};
\node[name=x, draw=black, fill=\colorx, circle, left= .5 of y] {$x$};
\node[name=v, draw=black, fill=\colorv, circle, above= .5 of x] {$v$};
\draw[-stealth] (y) -- (x);
\draw[-] (v) -- (x);
\draw[-stealth] (y) |- (v.east);
\end{tikzpicture}
\caption{\begin{tabular}[t]{@{}c}
  Proposal \\ $q(v, x; y)$
\end{tabular}}
\label{fig:distributions-proposal}
\end{subfigure}
\begin{subfigure}[b]{.35\linewidth}
\centering
\begin{tikzpicture}
\node[name=y, draw=black, fill=black, label={right:$y$}] {};
\node[name=x, draw=black, fill=black, label={left:$x$}, left= .5 of y] {};
\node[name=v, draw=black, fill=\colorv, circle, above= .5 of x] {$v$};
\draw[-stealth] (x) -- (v);
\draw[-stealth] (y) |- (v.east);
\end{tikzpicture}
\caption{\begin{tabular}[t]{@{}c}
  Aux. Proposal  \\ $r(v; x, y)$
\end{tabular}}
\label{fig:distributions-aux-proposal}
\end{subfigure}
\captionsetup{skip=4pt, belowskip=5pt}
\caption{Target, proposal, and auxiliary proposal distributions used for
interval estimators (\cref{alg:entropy-upper-general,alg:entropy-lower-general})
of the entropy $\H{Y} = -\E{\log p(Y)}$, $Y\,{\sim}\,p(y)$.}
\label{fig:distributions}

\begin{algorithm}[H]
\caption{Monte Carlo upper bound $\hat{H}_Y$ on $\H{Y}$}
\label{alg:entropy-upper-general}
\begin{algorithmic}[1]
\For{$i=1\dots n$}
  \State $(\widetilde{X}, Y) \sim p(x,y)$
  \For{$j=1\dots m$}
    \State $(V, X) \sim q(v, x; Y)$
    \State $\displaystyle t_{ij} \gets \log \frac{p(X, Y)r(V; X, Y)}{q(V, X; y)}$
  \EndFor
\EndFor
\State \Return $-\sum_{i=1}^n\sum_{j=1}^m t_{ij} / nm$
\end{algorithmic}
\end{algorithm}
\vspace{-.05cm}
\begin{algorithm}[H]
\caption{Monte Carlo lower bound $\check{H}_Y$ on $\H{Y}$}
\label{alg:entropy-lower-general}
\begin{algorithmic}[1]
\For{$i=1\dots n$}
  \State $(X'_1, Y) \sim p(x, y)$
  \State $(X'_{2:m}) \sim \begin{aligned}[t]
  &\textrm{\footnotesize Markov chain targeting}\ p(x\mid Y) \\[-5pt]
  &\textrm{\footnotesize starting at}\ X'_1\ \textrm{\footnotesize (optional step)}
  \end{aligned}$
  \For{$j=1 \dots m$}
    \State $V \sim r'(v; X'_j, Y)$
    \State $\displaystyle t_{ij} \gets -\log \frac{q'(V, X'_j; Y)}{p(X'_j, Y)r'(V; X'_j, Y)}$
  \EndFor
\EndFor
\State \Return $-\sum_{i=1}^n\sum_{j=1}^m t_{ij}/nm$
\end{algorithmic}
\end{algorithm}
\vspace{-.5cm}
\end{figure}

%% file: figures/mvn.tex

\begin{figure}[t]
\centering
\captionsetup[subfigure]{skip=5pt}
\begin{subfigure}{.5\linewidth}
\includegraphics[width=\linewidth]{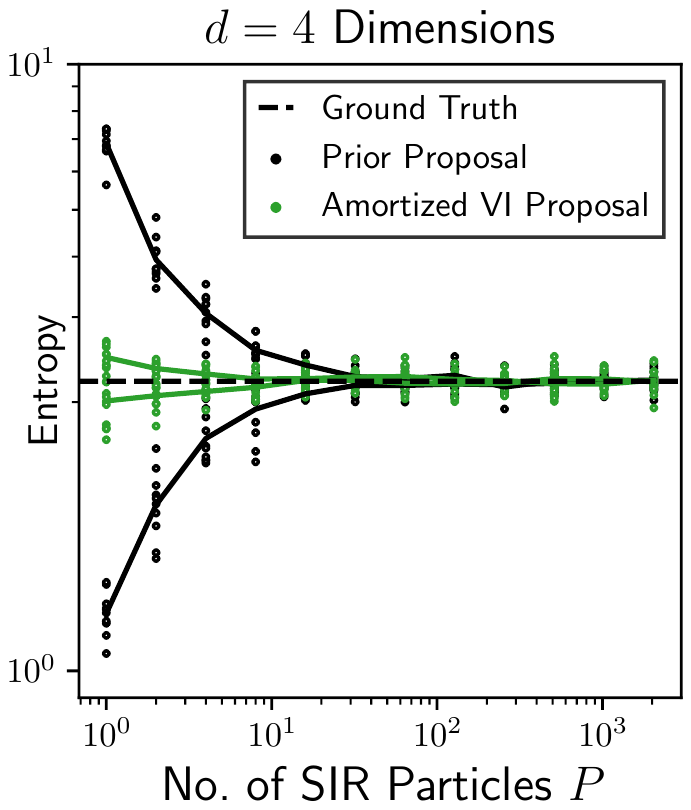}
\end{subfigure}\hfill
\begin{subfigure}{.5\linewidth}
\includegraphics[width=\linewidth]{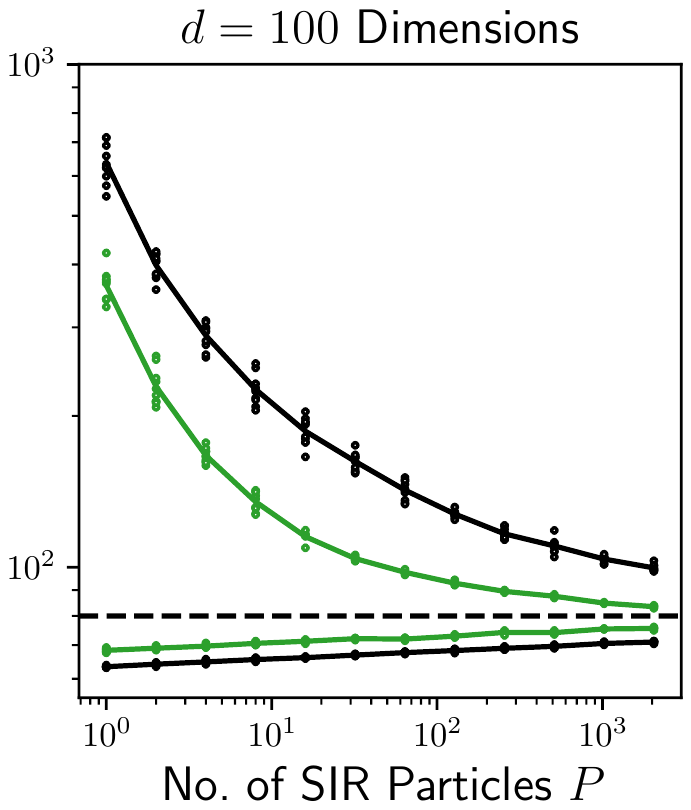}
\end{subfigure}
\captionsetup{skip=5pt,belowskip=-10pt}
\caption{Lower and upper bounds on the entropy of
$d/2$ dimensions $Y$ of a $d$-dimensional Gaussian $(X,Y)$
using \cref{alg:entropy-upper-general,alg:entropy-lower-general}
with the SIR scheme from~\cref{example:sir}.
The base proposals $q_0(x; y)$ are the prior and an amortized
variational approximation to the posterior that specifies a
separate regression for each dimension of $X$ given $Y$.
While the bounds converge to the ground truth (known
in closed form for Gaussians) as the number of SIR particles $P$ increases using both
proposals, the variational proposal is closer in ``exclusive'' and ``inclusive'' KL to the posterior,
resulting in a higher accuracy at each $P$.
At $d=100$, the lower bounds exhibit much lower bias and
variance as compared to the upper bounds, especially for small $P$.}
\label{fig:mvn}
\end{figure}

%% file: figures/alg-csmc.tex

\begin{figure}[t]
\setlength{\intextsep}{0pt}
\begin{algorithm}[H]
\caption{SMC Proposal $q(v, x; y)$}
\label{alg:smc-forward}
\algrenewcommand\algorithmicindent{1em}%
\begin{algorithmic}[1]
\Require \begin{tabular}{@{}l}
  Observation $y$ \\
  \end{tabular}
\Ensure Approximate sample $x$ from $p_T(x; y)$ and record $v$ of all
  sampled auxiliary random variables.
\State $x^i_0 \sim q_0(-; y)$ $(i = 1 \dots P)$
\State $w^i_0 \gets \tilde{p}_0(x^i_0;y) / q_0(x^i_0; y)$ $(i = 1 \dots P)$
\For{$t=1 \dots T$}
  \State $a^i_t \gets \mathrm{Categorical}(w^{1:P}_{t-1})$ $(i=1\dots P)$
\State $x^i_t \sim q_t( -; x^{a^i_t}_{t-1}, y)$ $(i=1\dots P)$
\State $w^i_t \gets \displaystyle\frac
      {\tilde{p}_t(x^i_t; y)l_{t-1}(x^{a^i_t}_{t-1}; x^i_t, y)}
      {\tilde{p}_{t-1}(x^{a^i}_{t}; y)q_t(x^i_t; x^{a^i_t}_{t-1}, y)}$
      $(i=1\dots P)$
      \label{algline:smc-forward-weight}
\EndFor
\State $I_T \sim \mathrm{Categorical}(w^{1:P}_T)$
\State \Return $(v, x) \defas ((I_T, a^{1:P}_{1:T}, x^{1:P}_{0:T}), x^{I_T})$
\end{algorithmic}
\end{algorithm}
\vspace{-.05cm}
\begin{algorithm}[H]
\caption{Auxiliary SMC Proposal $r(v; x, y)$}
\label{alg:smc-reverse}
\algrenewcommand\algorithmicindent{1em}%
\begin{algorithmic}[1]
\Require \begin{tabular}{@{}l}
  Observation $(x,y)$
  \end{tabular}
\Ensure Approximate sample $v$ of auxiliary variables
  generated by a run of~\cref{alg:smc-forward} that returned $x$.
\State $I_T \sim \mathrm{Uniform}(1 \dots P)$
\State $x^{I_T}_T \gets x$
\For{$t=T-1 \dots 0$}
  \State $I_t \sim \mathrm{Uniform}(1 \dots P)$
  \State $x^{I_t}_{t} \sim l_t(- ; x^{I_{t+1}}_{t+1}, y)$
  \State $a^{I_{t+1}}_{t+1} \gets I_t$
\EndFor
\State $x^i_0 \sim q_0(-; y)$ $(i = 1 \dots P; i \ne I_0)$
\State $w^i_0 \gets \tilde{p}_0(x^i_0; y) / q_0(x^i_0; y)$ $(i = 1 \dots P)$
\For{$t=1 \dots T$}
  \State $a^i_t \gets \mathrm{Categorical}(w^{1:P}_{t-1})$ $(i = 1 \dots P, i \ne I_t)$
\State $x^i_t \sim q_t(-; x^{a^i_t}_{t-1}) (i=1\dots P; i \ne I_t)$
\State $w^i_t \gets \displaystyle\frac
      {\tilde{p}_t(x^i_t; y)l_{t-1}(x^{a^i_t}_{t-1}; x^i_t, y)}
      {\tilde{p}_{t-1}(x^{a^i}_{t}; y)q_t(x^i_t; x^{a^i_t}_{t-1}, y)}$
      $(i=1\dots P)$
      \label{algline:smc-reverse-weight}
\EndFor
\State \Return $v \defas (I_{T}, a^{1:P}_{1:T}, x^{1:P}_{0:T})$
\end{algorithmic}
\end{algorithm}
\vspace{-.5cm}
\end{figure}

%% file: figures/hepar2.tex

\begin{figure*}[ht]

\captionsetup{aboveskip=10pt,belowskip=5pt}
\captionsetup[subfigure]{aboveskip=0pt,belowskip=0pt}
\begin{subfigure}[m]{.375\linewidth}
\includegraphics[width=\linewidth]{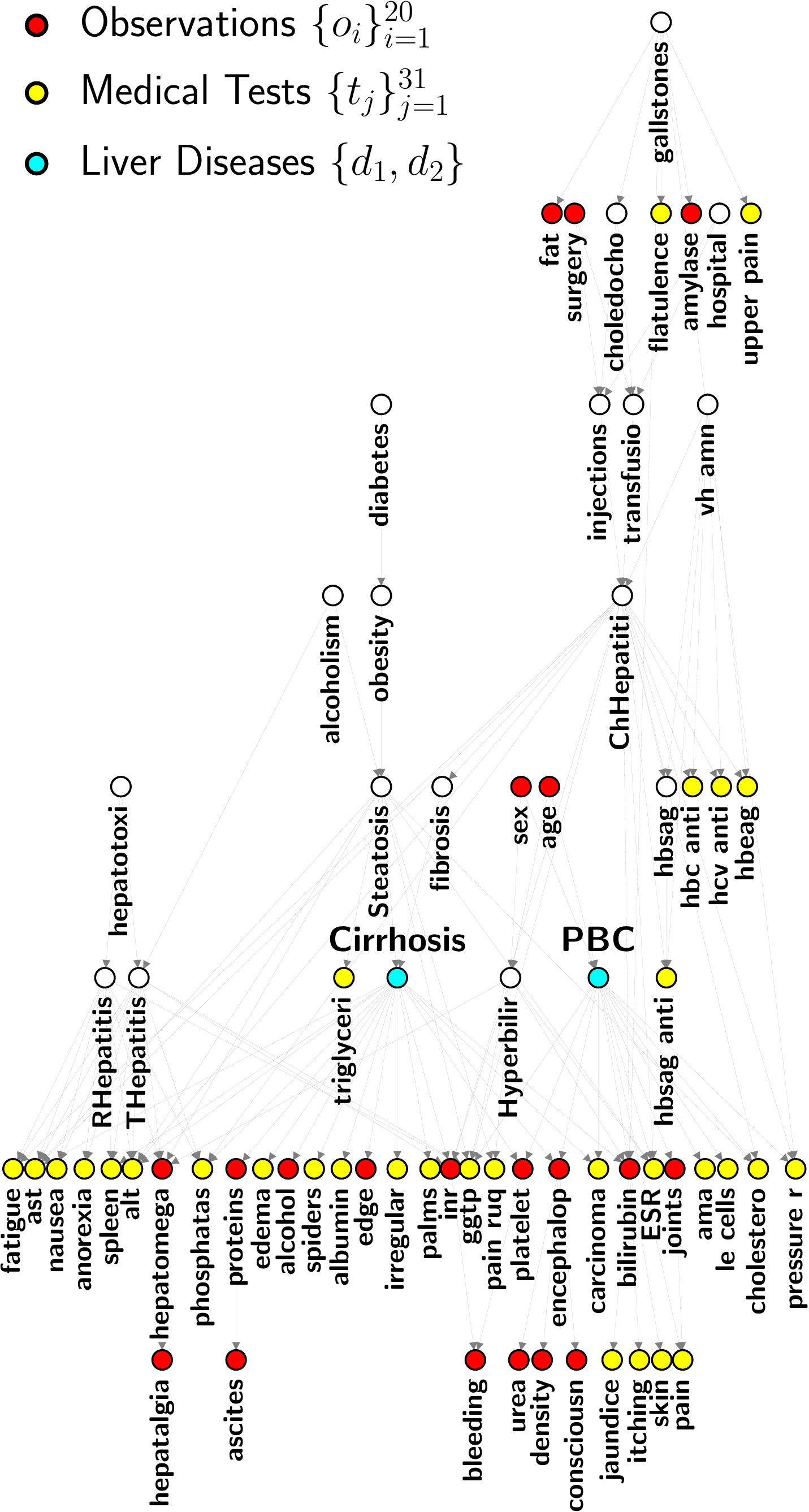}
\captionsetup{skip=5pt}
\caption{\mbox{HEPAR Liver Disease Model \citep{hepar2003}}}
\label{fig:hepar-network}
\end{subfigure}\hfill
\begin{subfigure}[m]{.6\linewidth}
\begin{subfigure}[t]{.05\linewidth}
\centering
\captionsetup{aboveskip=2pt, belowskip=0pt}
\caption*{}
\label{fig:hepar-dummy-pbc}
\begin{tikzpicture}
\node[name=a, fill=none] {};
\node[name=b, shape=coordinate, above=.1 of a] {};
\node[name=t, shape=coordinate, above=5.1 of b] {};
\draw[-stealth]
  (t)
  -- node[pos=.5,name=t,rotate=90,yshift=.25cm] {\footnotesize Decreasing Information Value}
  (b);
\end{tikzpicture}
\end{subfigure}\hfill
\begin{minipage}[t]{.93\linewidth}
\centering
\captionsetup{aboveskip=2pt, belowskip=0pt}
\caption{Top 10 medical tests for liver disease $d = $ PBC}
\label{tab:hepar-pbc-ranking}
\begin{tabular}{|l|Sl|l|}
\hline
Medical Test $t$
  & $\H{d \mid t, \set{o_i}_{i=1}^{20}}$
  & Prediction Error for $d$
\\\hline
ama             & 0.274 & 0.030 \\
ESR             & 0.346 & 0.094 \\
cholesterol     & 0.385 & 0.126 \\
ggtp            & 0.397 & 0.131 \\
carcinoma       & 0.402 & 0.144 \\
pain            & 0.404 & 0.148 \\
pressure ruq    & 0.410 & 0.174 \\
le cells        & 0.411 & 0.172 \\
irregular liver & 0.413 & 0.173 \\
edge            & 0.415 & 0.175 \\ \hline\hline
no test         & 0.418 & 0.175 \\
no test or obs  & 0.663 & 0.486 \\ \hline
\end{tabular}
\end{minipage}
\vspace{5pt}

\begin{subfigure}[t]{.05\linewidth}
\centering
\captionsetup{aboveskip=2pt, belowskip=10pt}
\caption*{}
\label{fig:hepar-dummy-cirrhosis}
\begin{tikzpicture}
\node[name=a, fill=none] {};
\node[name=b, shape=coordinate, above=.1 of a] {};
\node[name=t, shape=coordinate, above=5.1 of b] {};
\draw[-stealth]
  (t)
  -- node[pos=.5,name=t,rotate=90,yshift=.25cm] {\footnotesize Decreasing Information Value}
  (b);
\end{tikzpicture}
\end{subfigure}\hfill
\begin{minipage}[t]{.925\linewidth}
\centering
\captionsetup{aboveskip=2pt, belowskip=0pt}
\caption{Top 10 tests for liver disease $d = $ Cirrhosis}
\label{tab:hepar-cirrhosis-ranking}
\begin{tabular}{|l|Sl|l|}
\hline
Medical Test $t$
  & $\H{d \mid t, \set{o_i}_{i=1}^{20}}$
  & Prediction Error for $d$
\\\hline
irregular liver & 0.225 & 0.035 \\
edge            & 0.247 & 0.036 \\
spiders         & 0.253 & 0.037 \\
spleen          & 0.256 & 0.043 \\
palms           & 0.261 & 0.043 \\
carcinoma       & 0.261 & 0.049 \\
edema           & 0.262 & 0.056 \\
triglycerides   & 0.264 & 0.056 \\
albumin         & 0.268 & 0.068 \\
phosphatase     & 0.272 & 0.068 \\ \hline\hline
no test         & 0.290 & 0.070 \\
no test or obs  & 0.320 & 0.080 \\ \hline
\end{tabular}
\end{minipage}
\end{subfigure}
\caption{
Using EEVI to rank diagnostic medical tests
(yellow) in the HEPAR liver disease network by how informative they
are about diseases (blue) given a pattern of observations (red).
For both the PBC disease in \subref{tab:hepar-pbc-ranking} and
cirrhosis diseases in \subref{tab:hepar-cirrhosis-ranking}, conducting
tests that give lower conditional entropy $H(d \mid t, \set{o_i}_{i=1}^{20})$ of the disease
(i.e., higher conditional mutual information) results in lower prediction
errors about its presence or absence.
%
}
\label{fig:hepar}
\vspace{-.5cm}
\end{figure*}

%% file: figures/hepar2-runtime.tex
\begin{figure}[t]
\includegraphics[width=\linewidth]{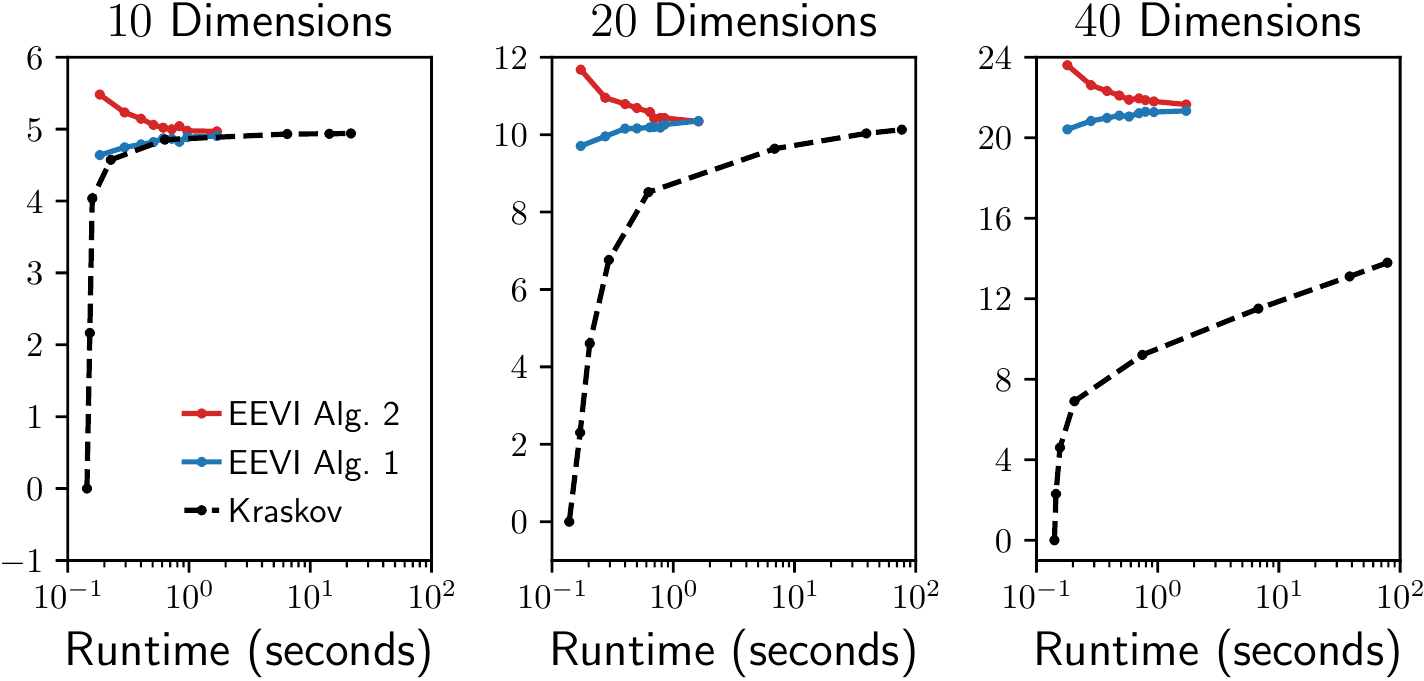}
\captionsetup{belowskip=-12pt}
\caption{Runtime for estimating entropies in the HEPAR network using
EEVI (\cref{alg:entropy-upper-general,alg:entropy-lower-general})
and the baseline nonparametric estimator of~\citet{kraskov2004} for varying
dimensionality.
Runtime increases with number of SIR samples $P$ for EEVI (see~\cref{eq:sir-p}) and
number of simulations from $p(x,y)$ for~\citeauthor{kraskov2004}
}
\label{fig:hepar-kraskov}
\end{figure}

%% file: figures/hepar2-variance.tex
\begin{figure}[t]
\centering
\begin{subfigure}[b]{\linewidth}
\centering
\includegraphics[width=\linewidth]{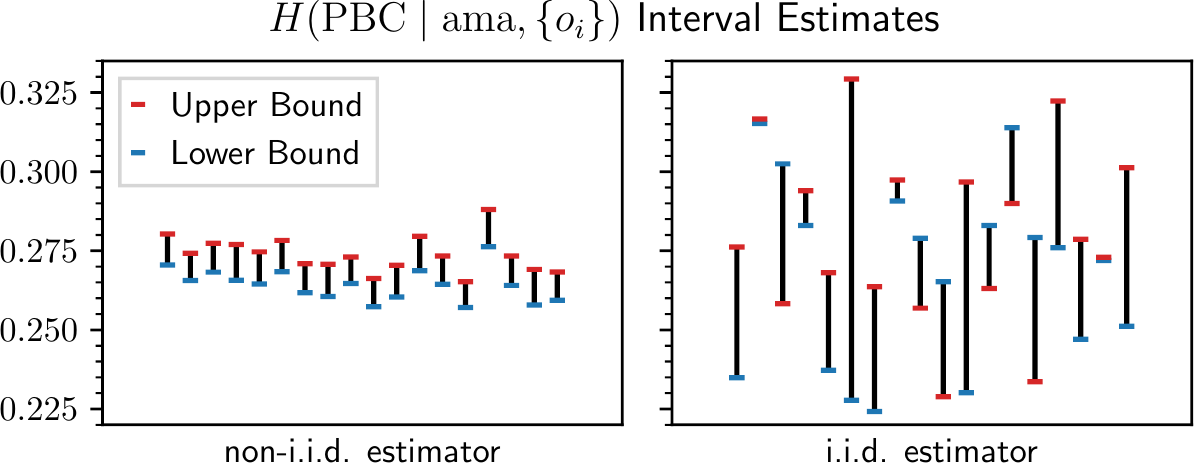}
\captionsetup{skip=5pt,belowskip=10pt}
\caption{36 Realizations of estimators of conditional entropy.}
\label{fig:hepar-variance-boxplot}
\end{subfigure}
\begin{subfigure}[b]{\linewidth}
\includegraphics[width=\linewidth]{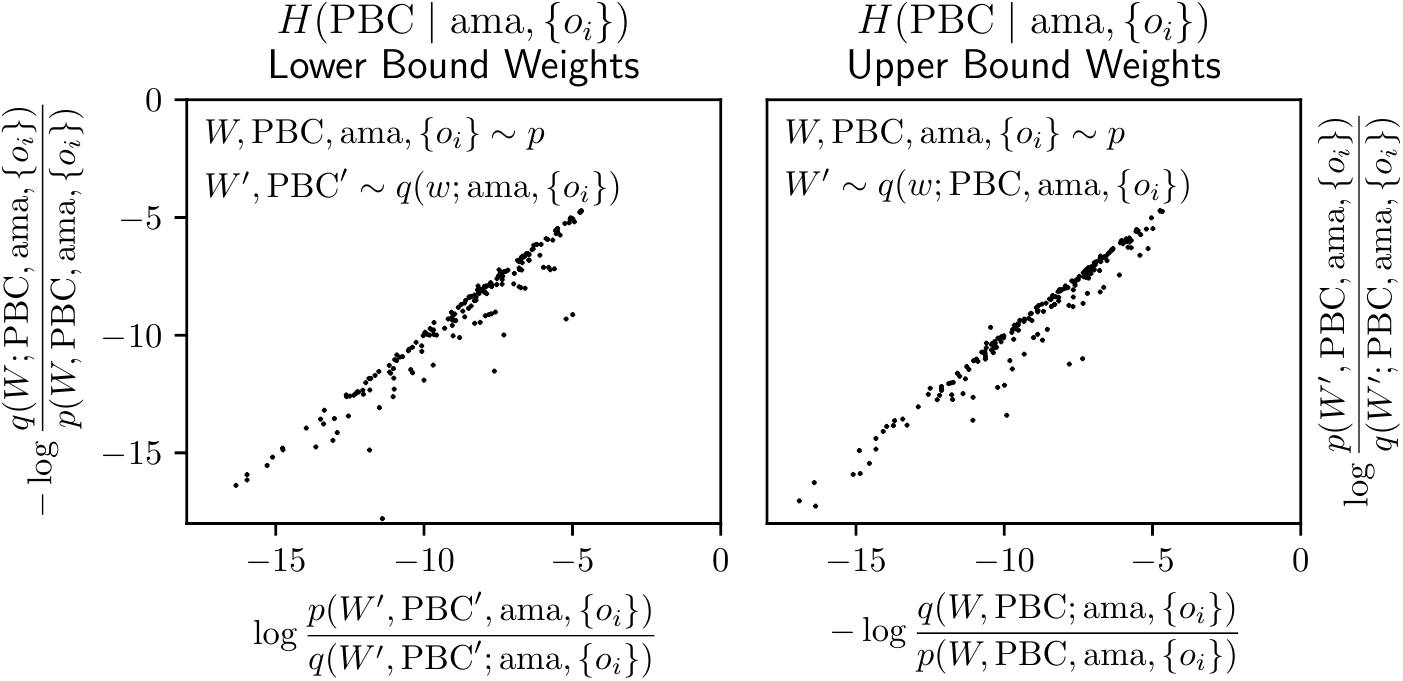}
\captionsetup{skip=5pt,belowskip=0pt}
\caption{200 random samples of log importance weights}
\label{fig:hepar-variance-scatter}
\end{subfigure}
\captionsetup{belowskip=-10pt}
\caption{\subref{fig:hepar-variance-boxplot} Reducing the variance of interval estimators of
$\H{\mathrm{PBC} \mid \mathrm{ama}, \set{o_i}}$ (\cref{tab:hepar-pbc-ranking}, row 1)
by non-i.i.d.\ sampling.
\subref{fig:hepar-variance-scatter}
As conditional entropy bounds~\cref{eq:ce-diff-lb,eq:ce-diff-ub} are average differences
(x-axes minus y-axes) of random weights that here are positively correlated,
computing differences using non-i.i.d.\ samples reduces variance by twice the covariance
(the variable $W$ refers to all other variables in the HEPAR network~\cref{fig:hepar-network}).}
\label{fig:hepar-variance}
\end{figure}

%% file: figures/diabetes2.tex

\begin{figure*}
\centering
\captionsetup[subfigure]{belowskip=0pt}
\centering
\begin{subfigure}[b]{.32\linewidth}
\centering
\includegraphics[width=\linewidth]{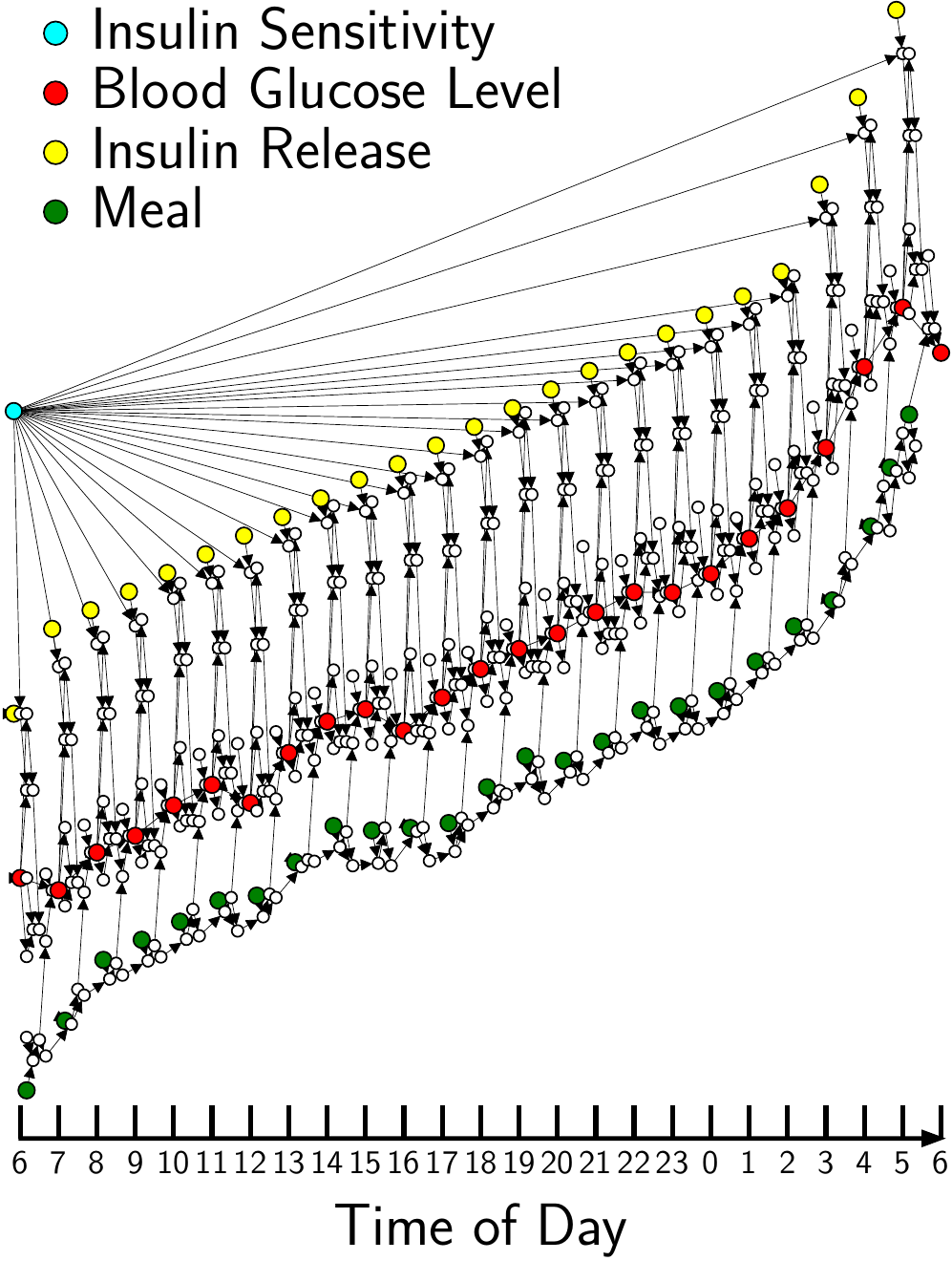}
\caption{Dynamic model of carbohydrate metabolism \citep{andreassen1991}}
\label{fig:diabetes-network}
\end{subfigure}\hfill
\begin{subfigure}[b]{.32\linewidth}
\includegraphics[width=\linewidth]{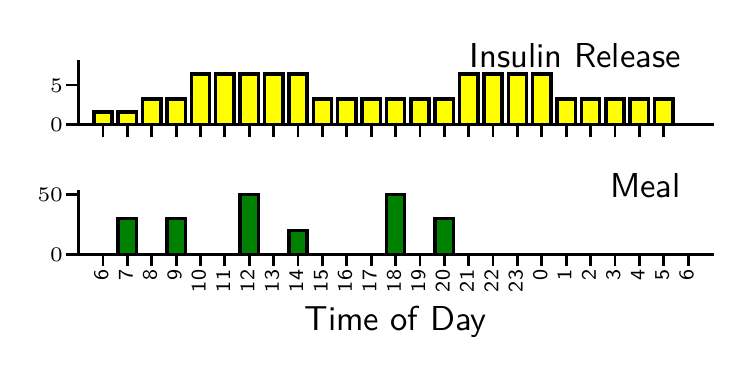}\vspace{-.25cm}
\includegraphics[width=\linewidth]{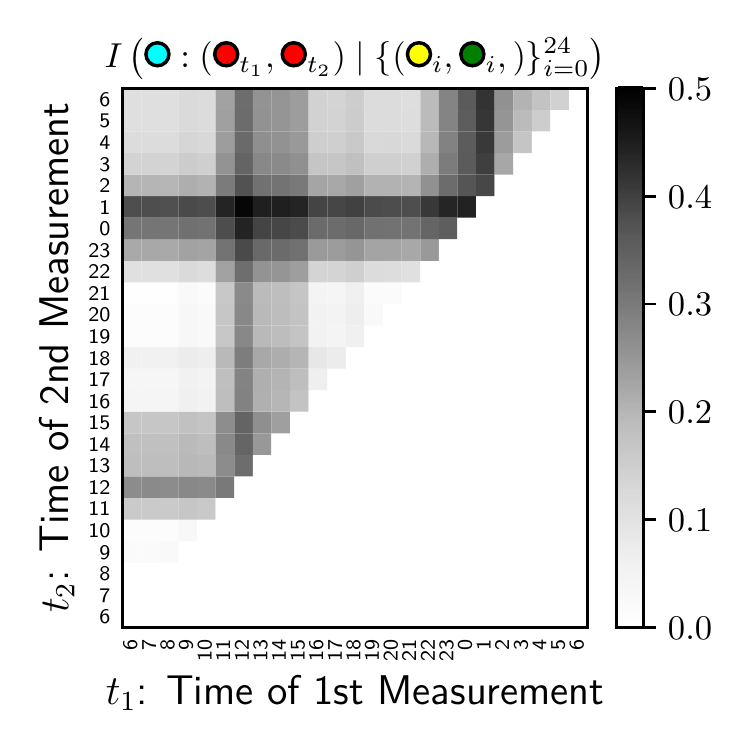}
\caption{Meal and Insulin Scenario 1}
\label{fig:diabetes-cmi-1}
\end{subfigure}\hfill%
\begin{subfigure}[b]{.32\linewidth}
\includegraphics[width=\linewidth]{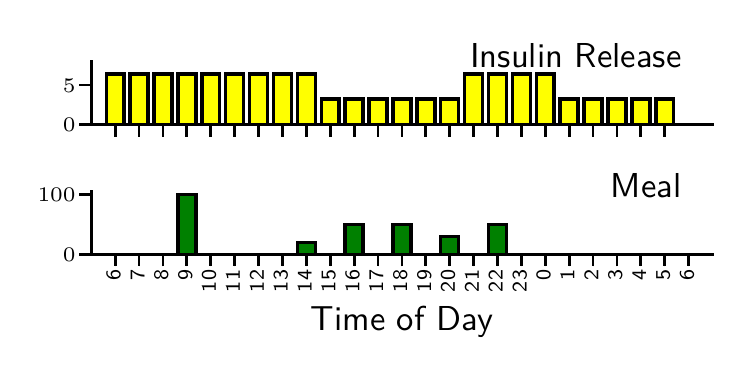}\vspace{-.25cm}
\includegraphics[width=\linewidth]{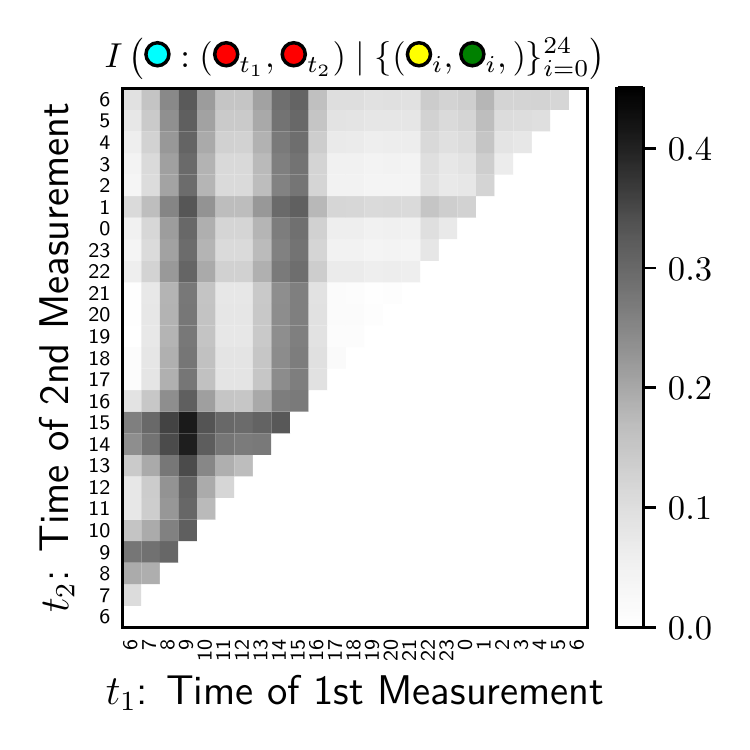}
\caption{Meal and Insulin Scenario 2}
\label{fig:diabetes-cmi-2}
\end{subfigure}
\caption{Inferring optimal pairs of times to measure blood glucose
level (red) that maximize information about a patient's latent
insulin sensitivity (blue).
Each heatmap in
\subref{fig:diabetes-cmi-1}--\subref{fig:diabetes-cmi-2} shows estimates from EEVI
of the conditional mutual information of insulin sensitivity with a pair of
blood glucose measurements for all pairs of times, under a certain
scenario of the patient's insulin release (yellow) and meal (green) schedule.
Each scenario has a different optimal pair of times to measure
blood glucose: 12pm/1am and 9am/3pm, respectively.
}
\label{fig:diabetes}
\end{figure*}

%% file: figures/gaps.tex

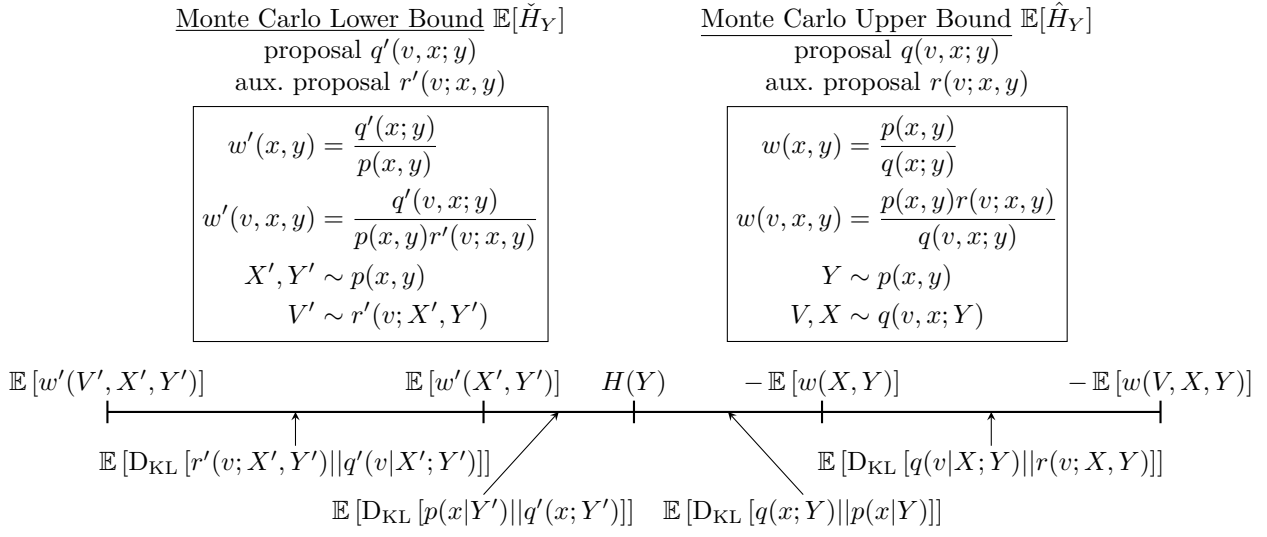
\begin{figure*}[!t]
\begin{tikzpicture}
\tikzstyle{box}=[draw=none, inner sep=1.5pt]

\node[name=xll, coordinate] at (0,0) {};
\node[name=xrr, coordinate, right=14 of xll] {};
\node[name=xm, coordinate] at ($(xll)!.5!(xrr)$) {};
\node[name=xlm, coordinate, left=2 of xm] {};
\node[name=xrm, coordinate, right=2.5 of xm] {};

\def\yoff{.15}
\draw[thick] (xll) -- (xrr);
\foreach \x in {xll, xrr, xm, xlm, xrm}
  \draw[thick] ([yshift=\yoff cm ]\x) -- ([yshift=-\yoff cm]\x);

\node[name=Hy, box, anchor=south, above=\yoff cm of xm] {$\H{Y}$};
\node[name=Elm, box, anchor=south, above=\yoff of xlm] {$\E{w'(X', Y')}$};
\node[name=Ell, box, anchor=south, above=\yoff cm of xll] {$\E{w'(V', X', Y')}$};
\node[name=Erm, box, anchor=south, above=\yoff cm of xrm] {$-\E{w(X,Y)}$};
\node[name=Err, box, anchor=south, above=\yoff cm of xrr] {$-\E{w(V,X,Y)}$};

\node[name=Kll, box, anchor=north, yshift=-.45cm] at ($(xll)!0.5!(xlm)$) {$\E{\KL{r'(v; X',Y')}{q'(v|X';Y')}}$};
\node[name=Klm, box, anchor=north, yshift=-1.1cm, xshift=-1cm] at ($(xlm)!0.5!(xm)$) {$\E{\KL{p(x|Y')}{q'(x;Y')}}$};
\node[name=Krm, box, anchor=north, yshift=-1.1cm, xshift=1cm] at ($(xm)!0.5!(xrm)$) {$\E{\KL{q(x;Y)}{p(x|Y)}}$};
\node[name=Krr, box, anchor=north, yshift=-.45cm] at ($(xrm)!0.5!(xrr)$)  {$\E{\KL{q(v|X;Y)}{r(v; X,Y)}}$};

\draw[-stealth] (Kll.north) -- ($(xll)!0.5!(xlm)$);
\draw[-stealth] (Klm.north) -- ($(xlm)!0.5!(xm)$);
\draw[-stealth] (Krm.north) -- ($(xrm)!0.5!(xm)$);
\draw[-stealth] (Krr.north) -- ($(xrm)!0.5!(xrr)$);

\node[name=Vl, draw=black, yshift=2.5cm,
  label={[align=center]above:{\underline{Monte Carlo Lower Bound} $\mathbb{E}[\check{H}_Y]$ \\ proposal $q'(v,x;y)$ \\ aux.\ proposal $r'(v;x,y)$}}
] at ($(xll)!0.5!(xm)$){$\begin{aligned}
  w'(x,y) &= \frac{q'(x;y)}{p(x,y)} \\
  w'(v,x,y) &= \frac{q'(v,x;y)}{p(x,y)r'(v;x,y)} \\
  X', Y' &\sim p(x,y) \\
  V' &\sim r'(v; X', Y')
  \end{aligned}$};

\node[name=Vr, draw=black, yshift=2.5cm,
  label={[align=center]above:{\underline{Monte Carlo Upper Bound} $\mathbb{E}[\hat{H}_Y]$ \\ proposal $q(v,x;y)$\\ aux.\ proposal $r(v;x,y)$}}
] at ($(xm)!0.5!(xrr)$){$\begin{aligned}
  w(x,y)   &= \frac{p(x,y)}{q(x;y)} \\
  w(v,x,y) &= \frac{p(x,y)r(v;x,y)}{q(v,x;y)} \\
  Y &\sim p(x,y) \\
  V, X &\sim q(v, x; Y)
  \end{aligned}$};
\end{tikzpicture}
\captionsetup{belowskip=30pt, aboveskip=30pt}
\caption[]{Characterization of the estimation gaps of the upper
and lower bounds $\hat{H}_Y$ and $\check{H}_Y$ in~\cref{alg:entropy-upper-general,alg:entropy-lower-general}.
}
\label{fig:gaps}
\end{figure*}